\documentclass{article}

\usepackage{microtype}
\usepackage{graphicx}
\usepackage{multirow}
\usepackage{subfigure}
\usepackage{booktabs} 
\usepackage{color}

\usepackage{resizegather}
\usepackage{hyperref}

 \usepackage[accepted]{icml2024}

\def \P{\mathbb{P}} 
\def \V {\mathbb{V}}
\def \E {\mathbb{E}}
\def \bfE {\mathbb{E}}

\newcommand{\indep}{\perp \!\!\! \perp}
\newcommand{\notindep}{\not \! \perp \!\!\! \perp}

\usepackage{amsmath}
\usepackage{amssymb}
\usepackage{mathtools}
\usepackage{amsthm}
\usepackage{bbm}

\usepackage[capitalize,noabbrev]{cleveref}

\theoremstyle{plain}
\newtheorem{theorem}{Theorem}[section]
\newtheorem{proposition}[theorem]{Proposition}
\newtheorem{lemma}[theorem]{Lemma}

\theoremstyle{definition}

\newtheorem{assumption}[theorem]{Assumption}
\theoremstyle{remark}

\usepackage[textsize=tiny]{todonotes}

\icmltitlerunning{Policy Learning for Balancing Short-Term and Long-Term Rewards}

\begin{document}

\twocolumn[
\icmltitle{Policy Learning for Balancing Short-Term and Long-Term Rewards}




\begin{icmlauthorlist}
\icmlauthor{Peng Wu}{btbu}
\icmlauthor{Ziyu Shen}{btbu}
\icmlauthor{Feng Xie}{btbu}
\icmlauthor{Zhongyao Wang}{alibaba}
\icmlauthor{Chunchen Liu}{alibaba}
\icmlauthor{Yan Zeng}{btbu}
\end{icmlauthorlist}


\icmlaffiliation{btbu}{School of Mathematics and Statistics, Beijing Technology and Business University}
\icmlaffiliation{alibaba}{LingYang, Alibaba Group, Hangzhou, China}

\icmlcorrespondingauthor{Yan Zeng}{yanazeng013@gmail.com}

\icmlkeywords{Machine Learning, ICML}

\vskip 0.3in
]


\printAffiliationsAndNotice{ }  

\begin{abstract}
Empirical researchers and decision-makers spanning various domains frequently seek profound insights into the long-term impacts of interventions. While the significance of long-term outcomes is undeniable, an overemphasis on them may inadvertently overshadow short-term gains. Motivated by this, this paper formalizes a new framework for learning the optimal policy that effectively balances both long-term and short-term rewards, where some long-term outcomes are allowed to be missing. In particular, we first present the identifiability of both rewards under mild assumptions. Next, we deduce the semiparametric efficiency bounds, along with the consistency and asymptotic normality of their estimators. We also reveal that short-term outcomes, if associated, contribute to improving the estimator of the long-term reward. Based on the proposed estimators, we develop a principled policy learning approach and further derive the convergence rates of regret and estimation errors associated with the learned policy. Extensive experiments are conducted to validate the effectiveness of the proposed method, demonstrating its practical applicability\footnote{Please note that our code is available at \url{https://github.com/YanaZeng/Short_long_term-Rewards}.}. 
\end{abstract}

\section{Introduction}

Empirical researchers and decision-makers usually seek profound insights into the long-term impact of interventions. For example, marketing professionals aim to understand how incentives influence customer behavior in the long term~\citep{Yang-etal2023}; IT companies explore the enduring effects of web page designs on user behavior~\citep{Hohnhold-etal2015-KDD}; economists examine the long-term impact of early childhood education on lifetime earnings~\citep{Chetty2011}; and medical practitioners investigate the impact of drugs on mortality in chronic diseases such as Alzheimer's and AIDS~\citep{Fleming1994}. Therefore, learning an optimal policy for personalized interventions to maximize long-term rewards holds significant practical implications.  

While long-term rewards are crucial, an exclusive focus on them may compromise short-term rewards, leading to ill-considered and suboptimal policies. Long-term effects can significantly differ from short-term effects~\citep{Kohavi-etal2012-KDD}, and in some cases, they may even exhibit opposing trends~\citep{chen2007criteria, ju2010criteria}. For instance, in video recommendation, the use of clickbait may initially boost click-through rates (CTR), but over the long term, it could lead to user churn and negatively impact a company's revenue~\citep{Wang2021-sigir}. 
In labor economics, individuals who participate in job training programs may initially experience a temporary decline in income but achieve elevated income levels and improved employment status in the following years~\citep{lalonde1986evaluating}. However, undue focus on future rewards would neglect the heavy pressure individuals can afford, which is unreasonable. Thus, achieving a balance between short-term and long-term rewards is desirable. 

This paper aims to learn the optimal policy that balances both long-term and short-term rewards. Policy learning refers to identifying individuals who should be given interventions based on their characteristics by maximizing rewards~\cite{murphy2003optimal}. Trustworthy policy learning necessitates that the learned policy also adheres to principles such as beneficence, non-maleficence, justice, and explicability~\cite{floridi2019establishing, thiebes2021trustworthy, kaur2022trustworthy}. However, the aspect of balancing short-term and long-term rewards in policy learning has not yet been explored.

Balancing short-term and long-term rewards presents some special challenges: akin to conventional policy learning methods, we need to address the confounding bias induced by factors that affect both treatment and short/long-term outcomes; long-term outcomes are hard to observe and often suffer from severe missing data due to extended follow-ups, drop-outs, and budget constraints~\citep{athey-etal2019, kallus2020role, hu2023longterm}; in addition, both short-term and long-term outcomes are post-treatment variables, with short-term outcomes influencing both the value and the missing rate of long-term outcomes~\citep{imbens2022long}. This is due to the fact that units are more likely to discontinue, experience churn, or fail to participate in follow-ups when short-term outcomes are not favorable. 

%
In this article, we propose a principled policy learning approach that effectively balances the short/long-term rewards. Specifically, we first define the short/long-term rewards and the optimal policy using the potential outcome framework~\cite{Rubin1974, Neyman1990} in causal inference. Then, we address confounding bias and the missingness of long-term outcomes by introducing two plausible assumptions, ensuring the identifiability of short/long-term rewards. 
To estimate short/long-term rewards for a given policy, we derive their efficient influence functions and semiparametric efficiency bounds. Building on this, we develop novel estimators that are shown to be consistent, asymptotically normal, and semiparametric efficient, i.e., they are optimal regular estimators in terms of asymptotic variance~\citep{Tsiatis-2006, Wu-etal-2024-Compare}. These results also reveal that short-term outcomes, if associated, contribute to the semiparametric efficiency bound of long-term reward. Additionally, the proposed estimators of short and long-term rewards enjoy the property of double robustness and quadruple robustness.  Finally, we learn the optimal policy based on the estimated short/long-term rewards, and further analyze the convergence rates of the regret and estimation error.



The contributions of this paper are summarized as follows.

~~$\bullet$ We propose and formulate a new setting of policy learning for balancing short-term and long-term rewards. The new setting has a wide range of application scenarios.

~~$\bullet$ We propose a principled policy learning approach for learning the optimal policy of balancing short-term and long-term rewards, by introducing plausible identifiability assumptions and novel estimation methods.

~~$\bullet$ We provide comprehensive theoretical analysis for the proposed approach, including identifiability results, semiparametric efficiency bounds, consistency and asymptotically normality of the estimators, as well as convergence rates of the regret and estimation error of the learned policy.

~~$\bullet$ We conduct extensive experiments to demonstrate the effectiveness of the proposed policy learning approach, verifying the superiority of taking both long-term and short-term rewards into consideration.

\section{Related Work}

{\bf Long-term causal effect estimation.} Exploring the long-term effect of the intervention has a wide range of applications in fields such as artificial intelligence, medical, clinical medicine, economics, and management~\citep{athey-etal2019}. 
 A salient feature of estimating long-term causal effects is that it takes a long time to collect long-term outcomes and is therefore difficult to observe. 
 To reduce the cost and time, and make timely decisions, researchers often look for easily observable short-term surrogates as substitutes for long-term outcomes, thereby transforming the problem of estimating long-term causal effects into estimating short-term causal effects~\citep{Yin-etal2020}. However, such strategies may suffer from the surrogate paradox~\citep{chen2007criteria}, i.e., treatment has a positive impact on a surrogate, which in turn has a positive effect on the outcome, but paradoxically, the treatment exhibits a negative effect on the outcome. 
 Subsequently, the selection of surrogates that matter has been studied for many years~\citep{prentice1989surrogate,frangakis2002principal,lauritzen2004discussion,chen2007criteria,ju2010criteria, Yin-etal2020}. Recently, inspired by the pioneering work of \citet{athey-etal2019}, several studies have emerged to identify and estimate the long-term causal effects using surrogates, such as \cite{kallus2020role, athey2020combining, chen2021semiparametric, cheng2021-WSDM,  hu2023longterm}. 
  Additionally, \citet{Yang2023} extend the work of \citet{athey-etal2019} to policy learning.

Unlike previous works that solely focus on long-term effects, we recognize that short-term effects are also of great importance in various applications. This paper considers short-term and long-term effects simultaneously.

{\bf Trustworthy policy evaluation and learning.} 
Policy learning aims to tailor treatments based on individual characteristics
~\cite{Kosorok+Laber:2019}. Early strategies for policy learning target maximizing the average rewards for an outcome~\citep{murphy2003optimal, Dudik2011-ICML, Zhao-etal2012, reg, Chen-etal2016, Wu-etal2022-framework, Wu-Tan2024, Wu-Han2024}.  
However, decisions made by algorithms to be trusted by humans have to take into account many other aspects besides maximizing rewards, such as beneficence, non-maleficence, harmlessness, autonomy, justice, and explicability~\cite{thiebes2021trustworthy, floridi2019establishing, kaur2022trustworthy, Wu-etal-2024-Harm}. Various causality-based metrics are proposed to evaluate the policy's trustworthiness~\cite{kusner2017counterfactual,nabi2018fair, chiappa2019path, wu2019pc, 2022NathanFacct, 2022nathan} and several trustworthy policy learning approaches are developed~\citep{Wang-etal2018, Kallus-Zhou2018, Qiu-etal2021, ben2022policy, Ding-etal2022, li2023trustworthy, Li-etal2023-EndtoEnd, Fang-etal2023}.

In this paper, we extend previous research and introduce a new setting that aims to learn the optimal policy for balancing short-term and long-term rewards, as well as   
 develop a principled approach. To the best of our knowledge, this is the first attempt to balance long and short-term rewards in policy learning under the causal inference framework. 
\section{Problem Formulation}

\subsection{Notation and Setup}

{\bf Notation}.  
Let $A$ be the binary treatment indicator, taking values 1 or 0 for the treated or control group, respectively. The vector $X \in \mathcal{X} \subset \mathbb{R}^p$ represents the observed pre-treatment features, and  $Y \in \mathcal{Y} \subset \mathbb{R}$ denotes the long-term outcome of interest. Additionally,  $S \in \mathcal{S} \subset \mathbb{R}$ denotes the short-term outcome that is informative about the long-term outcome $Y$ and measured after the treatment $A$. 

Under the potential outcome framework~\citep{Rubin1974, Neyman1990},  let $(S(1), Y(1))$ and $(S(0), Y(0))$ be the potential short-term and long-term outcomes with and without treatment, respectively.    
We assume that the actual short/long-term outcome corresponds to the potential outcome of the actual
treatment, i.e., $S = S(A)$ and $Y = Y(A)$, which implicitly implies the non-interference and consistency assumptions in causal inference~\cite{Imbens-Rubin2015}. 
Without loss of generality, we assume larger short/long-term outcomes are preferable. Each unit is assigned only one treatment, thus we always observe either $(S_i(0), Y_i(0))$ or $(S_i(1), Y_i(1))$ for unit $i$, which is also known as the fundamental problem of causal inference~\cite{Holland1986, Hernan-Robins2020}.

{\bf Setup.}  Long-term outcomes often suffer from missing due to factors such as long follow-ups, drop-out, and budget constraints. In contrast, it is easier to collect the short-term outcomes.  To mimic real-world application scenarios, we assume that all short-term outcomes $S$ are observable, while long-term outcomes $Y$ are allowed to be missing. 
Let $R \in \{0, 1\}$ be the indicator for observing the long-term outcome $Y$. Without loss of generality, the observed data consists of a subset $\{(X_i, A_i, S_i, Y_i, R_i = 1): i = 1, ..., n_1 \}$ with observed $Y$ and a subset   $\{(X_i, A_i, S_i, Y_i = \text{NA}, R_i = 0): i = n_1+1, ..., n_1 + n_0 \}$ with missing $Y$.  Let $n = n_0 + n_1$ and we assume the total $n$ units are a representative sample of the target population $\P$, denoting $\E$ as the expectation operator of $\P$. Table \ref{tab1} summarizes the data composition. 
The proposed method also works when there is no $Y$ missing, i.e., $R = 1$ for all units. 

\subsection{Formulation}



We here give formalization about learning an optimal policy that could strike a good balance between short-term and long-term rewards. 

Let $\pi: \mathcal{X}\to \{0, 1\}$ be a policy  that maps from the individual context $X=x$ to the treatment space $\{0, 1\}$. 
 For a given policy $\pi$, the policy values are defined as, 
 	\begin{align*}
	 \V(\pi; s) ={}& \E[ \pi(X) S(1) +(1-\pi(X)) S(0)  ], \\
	 \V(\pi; y) ={}& \E[ \pi(X) Y(1) +(1-\pi(X)) Y(0)  ],     
	\end{align*}
which are the expected short-term and long-term rewards given that $\pi$ is applied to the target population.  
 Then we formulate the goal as learning an optimal policy that satisfies 
 	\begin{align*}
 	     \begin{cases} \max_{\pi \in \Pi} \V(\pi; y) \\
	     		 \text{subject to }  \V(\pi; s)  \geq \alpha  	
 	     \end{cases}\text{or} \quad  	
 	     \begin{cases} \max_{\pi \in \Pi} \V(\pi; s) \\
	     		 \text{subject to }  \V(\pi; y)  \geq \alpha,  	
	     \end{cases}
 	\end{align*}	
 where $\alpha$ is a pre-specified threshold for minimum short-term or long-term rewards 
 and $\Pi$ is a pre-specified policy class.  The above two optimization problems can be expressed as 
	 \begin{equation}  \label{eq3}
	 	\max_{\pi \in \Pi}   \V(\pi; s)  + \lambda  \V(\pi; y),
	 \end{equation}
where $\lambda$ is a positive constant that controls the balance between short-term and long-term rewards. When $\lambda = 0$, Eq.\eqref{eq3} is equivalent to finding an optimal policy for maximizing the short-term reward alone; Conversely, when $\lambda = \infty$, it transforms into finding an optimal policy that maximizes the long-term reward alone.

\begin{table}[t!] 
\centering
\caption{Observed data, where $\checkmark$ and $\text{NA}$ mean observed and missing, respectively.}   \label{tab1}
\begin{tabular}{ cccc cc }    
\hline
$\textsc{Unit}$ &   $R$   &   $X$    &  $A$  &  $S$ & $Y$   \\
		\hline		
$1$ & $1$      &    \checkmark     &     \checkmark        &    \checkmark   & \checkmark                \\
$...$ &    $1$      &    \checkmark     &    \checkmark       &    \checkmark    & \checkmark              \\
$n_1$ &      $1$      &    \checkmark     &    \checkmark      &    \checkmark      &    \checkmark             \\
        \hline	
$n_1+1$ & $0$      &    \checkmark     &      \checkmark       &    \checkmark        & \text{NA}         \\
$...$ & $0$      &    \checkmark     &  \checkmark          &    \checkmark      &     \text{NA}     \\
$n$& $0$      &    \checkmark     &     \checkmark      &    \checkmark      &     \text{NA}         \\
\hline
\end{tabular}  
\end{table} 
\section{Optimal Policy and Challenges}

\subsection{Optimal Policy}
The optimal policy from maximizing Eq.\eqref{eq3} has an explicit form. Specifically, let $\tau_s(X) = \E[S(1) - S(0)| X]$ and $\tau_y(X) = \E[Y(1) - Y(0)| X]$  be the short-term and long-term causal effects conditional on $X$, then  we have 
\begin{gather}
\begin{align*}
	          & 	 \V(\pi; s)  + \lambda  \V(\pi; y) \\
		={}&  \E[  \pi(X) \{ S(1) - S(0) + \lambda(Y(1) - Y(0)) \} 
  + S(0) + \lambda Y(0)  ] \\
		={}& \E[  \pi(X)\{ \tau_s(X) + \lambda \tau_y(X) \} ] + \E[S(0) + \lambda Y(0)  ],
\end{align*}
\end{gather}
where the last equality follows from the law of iterated expectations. This implies the following Lemma \ref{lemma:optimal_policy}. 
\begin{lemma} \label{lemma:optimal_policy}
   The optimal policy  
	\begin{align*}
\pi^{*}_0(x) ={}& \arg \max_{\pi} \V(\pi; s)  + \lambda  \V(\pi; y)\\
={}& \arg \max_{\pi} \E[  \pi(X)\{ \tau_s(X) + \lambda \tau_y(X) \} ] \\
={}& 
\begin{cases}
 1,~~ \tau_s(x) + \lambda \tau_y(x) \geq 0 ~\\
 0,~~\tau_s(x) + \lambda \tau_y(x) < 0, \\   
\end{cases}
 	\end{align*}
 where $\arg \max$ is over all possible policies.  
\end{lemma}

Lemma \ref{lemma:optimal_policy} suggests that for a unit with $X=x$, the optimal policy recommends accepting treatment ($A=1$) if the sum of the weighted short-term and long-term causal effects, $\tau_s(x) + \lambda \tau_y(x)$, is positive; otherwise, it recommends not accepting treatment ($A=0$). More generally, if taking treatment has a cost $c$ and define $\V(\pi; s)$ and $\V(\pi; y)$ as    
\begin{align} \label{eq:cost}
\begin{split}
\E[ \pi(X) \{S(1)-c\} +(1-\pi(X)) S(0)  ], \\
\E[ \pi(X) \{Y(1)-c\} +(1-\pi(X)) Y(0)  ],    
\end{split}
\end{align}
respectively. Then the optimal policy becomes $\pi_0^*(x) = \mathbb{I}(\tau_s(x) + \lambda \tau_y(x) \geq c)$.  This aligns with our intuition and the goal of balancing short-term and long-term rewards.


\subsection{Challenges}
There are two main challenges in learning the optimal policy for balancing short-term and long-term rewards. 
\begin{itemize}
    \item  Confounding bias occurs when the treatment is not randomly assigned, and certain factors may affect both the treatment $A$ and the outcomes $(S, Y)$~\citep{Correa-etal2019}. In such cases, the effects of these factors become confounded with the effect of treatment, making it challenging to obtain unbiased estimators of short-term and long-term causal effects.
    \item The long-term outcome $Y$ is not missing completely at random, indicating a systematic difference between observed data (i.e., $R=1$) and missing data (i.e., $R=0$). Moreover, both short-term and long-term outcomes are post-treatment variables, with short-term outcomes influencing both the value and the missing rate of long-term outcomes \citep{imbens2022long}.  
 \end{itemize}

The identifiability problem arising from these two challenges will be addressed in Section \ref{sec-5}. Interestingly, in Section \ref{sec6-1}, we discover that short-term outcomes can assist in enhancing the estimation of long-term rewards, thereby transforming part of the second challenge into an advantage. {See the discussion below Theorem \ref{thm-EIF} and Proposition \ref{prop-efficiency} for more details.}

\section{Identifiability} \label{sec-5}

We present identifiability assumptions for the short-term reward $\V(\pi; s)$ and the long-term reward $\V(\pi; y)$.

\begin{assumption}[Strongly Ignorability] \label{assump5-1} \quad 

 (a) $(S(a), Y(a)) \indep A \mid X$ for $a = 0, 1$; 

 (b) $0 < e(x) \triangleq \P(A=1 \mid X = x) < 1$ for all $x \in \mathcal{X}$.  
\end{assumption}

 Assumption \ref{assump5-1}(a) states that $X$ includes all confounders that affect both the outcomes $(S, Y)$ and treatment $A$, i.e., there are no unmeasured confounders.   
 Assumption \ref{assump5-1}(b) asserts that units with any given values of the features have a positive probability of receiving each treatment option. Both of them are standard assumptions in causal inference \citep{Imbens-Rubin2015, Hernan-Robins2020}.  

 Assumption \ref{assump5-1} ensures the identifiability of the short-term reward $\V(\pi; s)$, which is given as  
  \begin{equation}  \label{eq4}   \V(\pi; s) = \E[ \pi(X) \mu_1(X) + (1 - \pi(X)) \mu_0(X) ],      \end{equation}
 where $\mu_a(X) = \E[S| X, A=a]$ for $a = 0, 1$.

 
  To identify the long-term reward $\V(\pi; y)$, we need to impose further assumptions on the missing mechanism of $Y$. 

 \begin{assumption}[Missing Mechanism] \label{assump5-2}  
 For $a = 0, 1$, 
 
 (a)  $R \indep Y(a) \mid X, S(a), A = a$; 
 
 (b) $0 < r(a, x, s) \triangleq \P(R =1| X = x, A=a, S = s)$. 
\end{assumption}
 
  Assumption \ref{assump5-2}(a) can be equivalently expressed as $R \indep Y\mid (X, S, A)$. It implies that $\P(R=1|X, S, A, Y) = \P(R=1|X,S,A)$, i.e., the observing indicator $R$ depends on only the feature $X$, the treatment $A$ and short-term outcome $S$. This assumption also guarantees that $\P(Y=y|X, S, A, R=1) = \P(Y=y|X, S, A, R=0)$, i.e., the distribution of the long-term outcome on the missing data and non-missing data are comparable after accounting for the observed variables $(X, A, S)$. 
   Consequently, we can use the non-missing data to make inferences about the missing long-term outcome. 
Assumption \ref{assump5-2}(b) assumes that each unit has a positive probability of being observed.

Different from the conventional missing mechanism assumption "$R \indep (S(a), Y(a)) \mid X$`` that $R$ depends solely on $X$,  
Assumption \ref{assump5-2}(a) is weaker and allows $R$ to depend on $(X, A, S)$, i.e., the missing mechanism relies not only on the covariates but also on the treatment and short-term outcomes. In addition,   Assumption \ref{assump5-2}(a) is more realistic and aligns with real-world scenarios. This is because units are more likely to drop out,  churn, or fail in follow-up when short-term outcomes $S$ are not desirable. 
Assumptions \ref{assump5-1}-\ref{assump5-2} ensure the identifiability of $\V(\pi; y)$, as shown in Proposition \ref{prop5-3} (See Appendix \ref{app-a} for proofs).  
\begin{proposition}[Identifiability of $\V(\pi; y)$] \label{prop5-3} Under Assumptions \ref{assump5-1}-\ref{assump5-2}, the long-term reward $\V(\pi; y)$ is identified as 
	  	 	\begin{align*}
		     \V(\pi; y) 
		     		={}&  \E[ \pi(X)  \tilde m_1(X, S)  + (1 - \pi(X)) \tilde m_0(X, S) ],
		 \end{align*}    
where $\tilde m_a(X, S) = \E[Y | X, S, A=a, R=1 ]$ for $a = 0, 1$.   
\end{proposition}

\section{Policy Learning for Balancing Short-Term and Long-Term Rewards} \label{sec6}

The proposed method consists of the following two steps: (a) policy evaluation, estimating the short-term and long-term rewards $\V(\pi; s)$ and $\V(\pi; y)$ for a given policy; (b) policy learning, solving the optimization problem \eqref{eq3} based on the estimated values of  $\V(\pi; s)$ and $\V(\pi; y)$.

\subsection{Estimation of Short-Term and Long-Term Rewards} \label{sec6-1}

To fully leverage the collected data, we aim to derive the efficient estimators of $\V(\pi; s)$ and $\V(\pi; y)$ by resorting to the semiparametric efficiency theory \citep{Tsiatis-2006}. An efficient estimator, often considered the optimal estimator (or gold standard), is the one that achieves the semiparametric efficiency bound---the smallest possible asymptotic variance among all regular estimators given the observed data \citep{Newey1990, vdv-1998}.

To derive efficient estimators, we initially calculate the efficient influence function and the semiparametric efficiency bound of  $\V(\pi; s)$ and $\V(\pi; y)$. 
For clarity, we summarize the nuisance parameters in Table \ref{tab2} that are utilized in the following theory and all of them can be identified from the observed data. 

\begin{table}[t!] 
\centering
\caption{Nuisance parameters in the efficiency influence functions of  $\V(\pi; s)$ and $\V(\pi; y)$.}   \label{tab2}
\setlength{\tabcolsep}{2pt}
\resizebox{1\linewidth}{!}
{\begin{tabular}{ ll }
\hline
$\textsc{Quantity}$ &   $\textsc{Description}$    \\
\hline 
$e(X)=\P(A=1|X)$,  &  propensity score     \\
$r(A, X, S) = \P[R=1 | X, S, A]$, & selection score  \\
$\mu_a(X)=\E(S|X, A=a)$,  &  regression function for $S$   \\
$m_a(X) = \E[Y | X, A=a, R=1 ]$, & regression function for $Y$ \\
$\tilde m_a(X, S) = \E[Y | X, S, A=a, R=1 ]$, & regression function for $Y$ \\
\hline 
\end{tabular}} 
\end{table}


\begin{theorem}[Efficiency Bounds of $\V(\pi; s)$ and $\V(\pi; y)$] \label{thm-EIF}  Under Assumptions \ref{assump5-1} and  \ref{assump5-2},  we have that 

(a) the efficient influence function of $\V(\pi; s)$ is $\phi_s - \V(\pi; s)$,  
	\begin{gather*}
	    \begin{align*}
	        \phi_s &={}   \{ \pi(X) \mu_1(X) + (1 - \pi(X)) \mu_0(X) \} \\ 
	         +{}& \frac{\pi(X) A(S - \mu_1(X))}{e(X)}  +  \frac{(1 - \pi(X)) (1-A)(S - \mu_0(X))}{1 - e(X)},
	\end{align*}
	\end{gather*}
the associated semiparametric efficiency bound is $\text{Var}(\phi_s)$.

(b) the efficient influence function of $\V(\pi; y)$ is $\phi_y - \V(\pi; y)$,  
\begin{gather*}
	\begin{align*}
	\phi_{y} &={} \{ \pi(X) m_1(X) + (1 - \pi(X)) m_0(X) \} \\
	  +{}&  \frac{\pi(X) A R (Y - \tilde m_1(X, S))}{e(X) r(1, X, S)}  + \frac{\pi(X) A  ( \tilde  m_1(X, S) - m_1(X))}{e(X)}   \\
	  +{}&   \frac{(1 - \pi(X))(1-A) R (Y - \tilde  m_0(X, S))}{(1-e(X)) r(0, X, S)}   \\
	    +{}&  \frac{(1 - \pi(X))(1-A) ( \tilde m_0(X, S) - m_0(X)) }{1 - e(X)},
	\end{align*}	
\end{gather*}
the associated semiparametric efficiency bound is $\text{Var}(\phi_y)$. 
\end{theorem}

Theorem \ref{thm-EIF} presents the efficient influence functions of $\V(\pi; s)$ and $\V(\pi; y)$, which are crucial for constructing efficient estimators of the short-term and long-term rewards. 
From Theorem \ref{thm-EIF}(b), $S$ plays a role in $\phi_y$ through $\tilde m_a(X, S)$. If $S \indep Y | X$, then  $\tilde m_a(X, S) = m_a(X)$ under Assumptions \ref{assump5-1} and \ref{assump5-2},  and the role of $S$ vanishes. Proposition \ref{prop-efficiency} (See Appendix \ref{app-b} for proofs) further demonstrates it from the perspective of semiparametric efficiency bound.  

\begin{proposition} \label{prop-efficiency} 
 Under the conditions in Theorem \ref{thm-EIF}, if $S$ is associated with $Y$ given $X$, then the semiparametric efficiency bound of $\V(\pi; y)$ is lower compared to the case where $S \indep Y | X$, and the magnitude of this difference is
 \begin{align*}
    & \E\left[ \pi(X) \frac{ (1 - r(1, X, S)) \cdot  (\tilde m_1(X, S) -  m_1(X))^2}{e(X) r(1, X, S)} \right ] + \\
     {}& \E\left[ (1 - \pi(X)) \frac{ (1 - r(0, X, S)) \cdot (\tilde m_0(X, S) - m_0(X))^2}{(1-e(X)) r(0, X, S)}   \right ].
 \end{align*}  
 \end{proposition}
{Proposition \ref{prop-efficiency} shows that if $S$ is correlated with $Y$ given $X$ (i.e., $S \not \indep Y | X$), the efficiency bound of $\mathbb{V}(\pi; y)$ is lower compared to the case where $S \indep Y | X$. This demonstrates the theoretical role $S$ plays in estimating the long-term reward. 
Furthermore, by a similar proof as Theorem \ref{thm-EIF} and Proposition \ref{prop-efficiency}, we can show that the efficiency bound of $\mathbb{V}(\pi; y)$ is lower when incorporating $S$ compared to not using $S$. This provides additional evidence of the significant role $S$ plays in estimating $\mathbb{V}(\pi; y)$.}


Next, we propose the efficient estimators of $\V(\pi; s)$ and $\V(\pi; y)$. For simplicity,  we let $Z = (X, A, S, Y)$ and write $\phi_s$ and $\phi_y$ in Theorem \ref{thm-EIF} as 
	\[   \phi_s = \phi_s(Z; e, \mu_0, \mu_1),  \phi_y = \phi_y(Z; e, r,   m_0, m_1, \tilde m_0, \tilde m_1)      \]
to highlight their dependence on intermediate quantities $(e, \mu_0, \mu_1)$ and $(e, r,  m_0, m_1, \tilde m_0, \tilde m_1)$.

Denote $\hat e(x)$, $\hat \mu_a(x)$, $\hat{m}_a(x)$, $\hat{\tilde m}_a(x, s)$, and $\hat r(a, x, s)$ for $a=0,1$ as the estimators of $e(x)$, $\mu_a(x)$, $m_a(x)$, $\tilde m_a(x,s)$, and $r(a,x,s)$ respectively, using the sample-splitting \cite{wager2018estimation, Chernozhukov-etal-2018} technique (See Appendix \ref{app-c} for details). 
From Theorem \ref{thm-EIF}, it is natural to define the estimators of $\V(\pi; s)$ and $\V(\pi; y)$ as 
\begin{align} \label{eq5}
\begin{split}
\hat   \V(\pi; s)  ={}&  \frac 1 n  \sum_{i=1}^n \phi_s(Z_i; \hat e, \hat \mu_0, \hat \mu_1),  \\
\hat  \V(\pi; y) ={}&  \frac 1 n  \sum_{i=1}^n \phi_{y}(Z_i; \hat e, \hat r, \hat m_0, \hat m_1, \hat{\tilde m}_0, \hat{\tilde m}_1).     
\end{split}
\end{align}

\begin{proposition}[Unbiasedness] \label{prop6-2} We have that 

(a) (Double Robustness).  $\hat  \V(\pi; s)$ is an unbiased estimator of $ \V(\pi; s)$ if one of the following conditions is satisfied: 
 \begin{itemize}
     \item[(i)] $\hat e(x) =e(x)$, i.e., $\hat e(x)$ estimates $e(x)$ accurately;  
     \item[(ii)] $\hat \mu_a(x) = \mu(x)$ i.e., $\mu_a(x)$ estimates $\mu_a(x)$ accurately. 
 \end{itemize}

(b) (Quadruple Robustness). $\hat  \V(\pi; y)$ is  an unbiased   estimator of $ \V(\pi; y)$ if one of the following conditions is satisfied: 
	\begin{itemize}
		\item[(i)] $\hat e(x) = e(x)$ and $\hat {\tilde m}_a(x, s) =\tilde m_a(x, s)$; 
		\item[(ii)] $\hat e(x) = e(x)$ and $\hat r(a, x, s) = r(a, x, s)$;    

  \item[(iii)] $\hat m_a(x) = m_a(x)$ and $\hat {\tilde m}_a(x, s) = \tilde m_a(x, s)$; 
		\item[(iv)] $\hat m_a(x) = m_a(x)$ and $\hat r(a, x, s) = r(a, x, s)$. 
	\end{itemize}
\end{proposition}

Proposition \ref{prop6-2}(a) (See Appendix \ref{app-b} for proofs) shows the double robustness of $\hat \V(\pi; s)$, i.e., it is unbiased if either the propensity score or the regression functions can be accurately estimated. Similarly, 
Proposition \ref{prop6-2}(b) demonstrates the quadruple robustness of $\hat \V(\pi; y)$. These properties provide protection against inaccuracies of estimated intermediate quantities. Furthermore, 
the proposed estimators $\hat \V(\pi; s)$ and $\hat   \V(\pi; y)$ are efficient under some mild conditions, please see Theorem \ref{Asy} for more details.   



\subsection{Learning the Optimal Policy}  \label{sec6-2}

Let $\pi^* = \arg \max_{\pi \in \Pi}  \V(\pi; s)  + \lambda  \V(\pi; y)$ be the target policy, which equals to $\pi^*_0$ in Lemma \ref{lemma:optimal_policy} if $\pi^*_0 \in \Pi$; otherwise, they may not be equal, and their difference is the systematic error induced by limited hypothesis space of $\Pi$.        

Let $\hat \pi^*$ be the learned policy of $\pi^*$, derived by optimizing the estimated $U(\pi) \triangleq \V(\pi; s)  + \lambda  \V(\pi; y)$, i.e., 
\begin{equation}  \label{eq6}
	 \hat \pi^* =	\arg \max_{\pi \in \Pi}   \hat \V(\pi; s)  + \lambda  \hat \V(\pi; y) \triangleq \arg \max_{\pi \in \Pi} \hat U(\pi), 
	 \end{equation}
where $\hat \V(\pi; s)$ and $\hat \V(\pi; y)$ are defined in Eq.\eqref{eq5}. 

Next, we explore the properties of $\hat \pi^*$, which depend on the asymptotic properties of $ \hat \V(\pi; s)$ and $\hat \V(\pi; y)$.

\begin{theorem}[Asymptotic Properties]\label{Asy} We have that 
   
(a) if $|| \hat e(x) - e(x)  ||_2 \cdot  || \hat \mu_a(x) - \mu_a(x) ||_2   = o_{\P}(n^{-1/2})$ for all $x\in\mathcal{X}$ and $a\in \{0,  1\}$,  then $\hat \V(\pi; s)$ is a consistent  estimator of $\V(\pi; s)$, and satisfies   
    \[ \sqrt{n} \{\hat \V(\pi; s) -  \V(\pi; s) \}     \xrightarrow{d} N(0,  \sigma_s^2),  \]
  where $\sigma_s^2 = \text{Var}(\phi_s)$ is the semiparametric efficiency bound of $\V(\pi; s)$, and  $\xrightarrow{d}$ means convergence in distribution.

(b) if $|| \hat e(x) - e(x)  ||_2 \cdot  || \hat m_a(x) - m_a(x) ||_2   = o_{\P}(n^{-1/2})$ and $|| \hat r(a, x, s) - r(a, x, s)  ||_2 \cdot  || \hat {\tilde m}_a(x, s) - \tilde m_a(x, s) ||_2   = o_{\P}(n^{-1/2})$  for all $x\in\mathcal{X}$, $a\in \{0,  1\}$ and $s \in \mathcal{S}$, then $\hat \V(\pi; y)$ is a consistent estimator of $\V(\pi; y)$, and satisfies 
    \[ \sqrt{n} \{\hat \V(\pi; y) -  \V(\pi; y) \}     \xrightarrow{d} N(0,  \sigma_y^2),  \]
  where $\sigma_y^2$ is the semiparametric efficiency bound of $\V(\pi; y)$.  
\end{theorem}

Theorem \ref{Asy} (See Appendix \ref{app-b} for proofs) establishes the consistency and asymptotic normality of proposed estimators $\hat \V(\pi; s)$ and $\hat \V(\pi; y)$. Additionally, these estimators are efficient, achieving the semiparametric efficiency bounds. Also, $\hat U(\pi)$ is the efficient estimator of $U(\pi)$ by the linearity of the influence function.  These desired properties hold under mild conditions concerning the convergence rate of estimated nuisance parameters, commonly used in causal inference \citep{Chernozhukov-etal-2018, Semenova-Chernozhukov}. These conditions are easily satisfied, provided that the nuisance parameters are estimated at the slower rate of $n^{-1/4}$, a criterion achievable by many flexible machine learning methods.  %

Based on the results in Theorem \ref{Asy}, we further explore the convergence rates of $U(\pi^*)- U(\hat \pi^*)$ and $U(\pi^*)-\hat U(\hat \pi^*)$, 
which are the regret of the learned policy, and error of the estimated reward of the learned policy, respectively. 

 \begin{proposition}[Regret and Estimation Error] \label{prop6-4} Suppose that for all $\pi\in\Pi$, $\pi(x)=\pi(x;\theta)$ is a continuously differentiable and convex function with respect to $\theta$, under the conditions in Theorem \ref{Asy}, we have
 
 (a) The expected reward of the learned policy is consistent, and $U(\hat \pi^*) - U(\pi^*) = O_{\mathbb{P}}(1 /\sqrt{n})$;

(b) The estimated reward of the learned policy is consistent, and 
$\hat U(\hat \pi^*) - U(\pi^*)= O_{\P}(1/\sqrt{n})$.  
\end{proposition}

Proposition \ref{prop6-4} (See Appendix \ref{app-b} for proofs) demonstrates that both  
 the regret of the learned policy $U(\hat \pi^*) - U(\pi^*)$ and estimation error of the estimated reward $\hat U(\hat \pi^*) - U(\pi^*)$ exhibit a convergence rate of order $1/\sqrt{n}$ for parametric policy classes.These results hold under mild assumptions commonly adopted in practice \cite{puterman2014markov, sutton2018reinforcement}. 

\subsection{Extension to Multi-Valued Treatments and Multiple Short-Term and Long-Term Rewards}

The proposed method could be readily extended to the case of multi-valued treatments and multiple short-term rewards. 
Specifically, suppose that the treatment $A$ takes values in $\{1, ..., K\}$ and we have $M$ short-term rewards and $J$ long-term rewards,  denoted as $(S_1, ..., S_M)$ and $(Y_1, ..., Y_J)$, respectively. Using the potential outcome framework, we define $(S_1(a), ..., S_M(a), Y_1(a), ..., Y_J(a))$ as the potential short/long-term outcomes if treatment $A$ had been set to $a$. The short-term and long-term rewards are defined as 
$\mathbb{V}(\pi; s_1) = \mathbb{E}[\sum_{a=1}^K S_1(a) ]$, $\cdots$, $\mathbb{V}(\pi; s_M) = \mathbb{E}[\sum_{a=1}^K S_M(a) ]$, and $\mathbb{V}(\pi; y_1) = \mathbb{E}[\sum_{a=1}^K Y_1(a) ]$, $\cdots$, $\mathbb{V}(\pi; y_J) = \mathbb{E}[\sum_{a=1}^K Y_J(a)$.

Then we can formulate the goal by learning an optimal policy that maximizes $$\sum_{j=1}^J \lambda_{m} \mathbb{V}(\pi; y_j) + \sum_{m=1}^M \beta_{m} \mathbb{V}(\pi; s_m),$$ where $(\lambda_1, ..., \lambda_J)$ and $(\beta_1, ..., \beta_M)$ are the weights that balance the short-term and long-term rewards. In addition, all the theoretical analysis in Sections \ref{sec-5}, \ref{sec6-1}, and \ref{sec6-2} can also be extended to such cases.

\section{Experiments}
\begin{table*}[t!]
\centering
\caption{Comparison of the baselines, \textsc{Naive-S} (maximizing short-term rewards alone), \textsc{Naive-Y} (maximizing long-term rewards alone), and our method (\textsc{OURS}) with DM, OR, IPW and our proposed (\textsc{PRO.}) estimators. They are reported
in terms of the rewards, welfare changes, and policy errors on \textsc{Jobs} and \textsc{PRODUCT}. Different balance factors are employed for the estimation and evaluation, $\lambda=0, 0.5, 1$, where the expected short-term and long-term rewards are estimated by outcome regression and multi-layer perceptron regression methods. Higher reward/$\Delta$W and lower error mean better performance.}
\begin{sc}
\resizebox{0.9 \linewidth}{!}{
\begin{tabular}{l | lll | lll | lll}
  \toprule
   \multicolumn{1}{c|}{ \textsc{JOBS}}  
   & \multicolumn{3}{c|}{Short-term metrics}  
   & \multicolumn{3}{c|}{\textbf{Balanced metrics}}
   & \multicolumn{3}{c}{Long-term metrics}\\
  \midrule
  \multicolumn{1}{l|}{Methods}   & 
  \multicolumn{1}{c}{Reward} & 
  \multicolumn{1}{c}{$\Delta$W} & \multicolumn{1}{c|}{ error} & 
  \multicolumn{1}{c}{Reward} & 
  \multicolumn{1}{c}{$\Delta$W} & \multicolumn{1}{c|}{error} &
  \multicolumn{1}{c}{Reward} & 
  \multicolumn{1}{c}{$\Delta$W} & \multicolumn{1}{c}{error} \\  
  \midrule
 \multirow{1}{*}{DM (NAIVE-S)}
  & \textbf{2011.2}&\textbf{736.4}&\textbf{0.295}&2000.0&\textbf{730.8}&0.418&-11.2&-11.2& 0.502\\
  \multirow{1}{*}{DM (NAIVE-Y) } 
&1554.6&279.9&0.572&1885.7&445.4&0.444&\textbf{331.0}&\textbf{331.0}&\textbf{0.423}\\
  \multirow{1}{*}{DM (OURS)}
&1811.4&536.6&0.439&\textbf{2069.6}&665.7&\textbf{0.404}&258.2&258.2&0.441 \\
   \multirow{1}{*}{OR (NAIVE-S)}
   &1487.6&212.8& 0.543&2631.0& 200.5&0.530&\textbf{1143.4}&\textbf{-24.6}&\textbf{0.506}\\
  \multirow{1}{*}{OR (NAIVE-Y) } 
  &\textbf{1616.0}&\textbf{341.2}&\textbf{0.489}&\textbf{2685.8}&\textbf{292.1}&\textbf{0.518}&1069.8& -98.1&0.524\\
  \multirow{1}{*}{OR (OURS)}
    &1559.8& 285.0&0.505&2661.5&251.9&0.522&1101.7&-66.3& 0.515 \\
  \multirow{1}{*}{IPW (NAIVE-S)}
  &\textbf{1654.6}&\textbf{379.8}&\textbf{0.505}&2782.8&\textbf{359.9}&0.497&1128.2&-39.7&0.508 \\
  \multirow{1}{*}{IPW (NAIVE-Y) } 
  &1575.3&300.6& 0.531&2807.9&332.9& 0.493&\textbf{1232.6}&\textbf{64.6}&\textbf{0.484}\\
  \multirow{1}{*}{IPW (OURS)}
  &615.3&340.6&0.514&\textbf{2809.2}&353.5&\textbf{0.491}&1193.9&25.9&0.493   \\
  \multirow{1}{*}{Pro. (NAIVE-S)}
  & \textbf{1698.0} & \textbf{423.2}&\textbf{0.465}&2845.4& 413.0 & 0.484 & 1147.3 & -20.6 & 0.504 \\
  \multirow{1}{*}{Pro. (NAIVE-Y)}
  &1605.8& 331.1&0.507&2866.2& 377.3&0.479 &\textbf{1260.4}&\textbf{92.4}&\textbf{0.476}\\
  \multirow{1}{*}{Pro. (OURS)} 
  &1657.2& 382.5& 0.481&\textbf{2895.4}&\textbf{417.6}&\textbf{0.473}&1238.1&70.2& 0.482\\
  \midrule
  \midrule
\multicolumn{1}{c|}{ \textsc{PRODUCT}}  
   & \multicolumn{3}{c|}{Short-term metrics}  
   & \multicolumn{3}{c|}{\textbf{Balanced metrics}}
   & \multicolumn{3}{c}{Long-term metrics}\\
  \midrule
  \multicolumn{1}{l|}{Methods}   & 
  \multicolumn{1}{c}{Reward} & 
  \multicolumn{1}{c}{$\Delta$W} & \multicolumn{1}{c|}{ error} & 
  \multicolumn{1}{c}{Reward} & 
  \multicolumn{1}{c}{$\Delta$W} & \multicolumn{1}{c|}{error} &
  \multicolumn{1}{c}{Reward} & 
  \multicolumn{1}{c}{$\Delta$W} & \multicolumn{1}{c}{error} \\  
  \midrule
  \multirow{1}{*}{DM (NAIVE-S)}
  & 2714.8& 473.7& 0.549&4868.8& 447.8& 0.529&\textbf{2153.9}& -51.7&0.503\\
  \multirow{1}{*}{DM (NAIVE-Y)}  
  &2763.9& 522.8& 0.555&4898.2& 487.1& 0.525&2134.2&-71.4&0.505 \\
  \multirow{1}{*}{DM (OURS)}
  & \textbf{2806.5}&\textbf{563.6}&\textbf{0.532}&\textbf{4975.4}&\textbf{534.1}&\textbf{0.518}&2131.2&\textbf{-21.2}&\textbf{0.498}\\
  \multirow{1}{*}{OR (NAIVE-S)}
  & \textbf{2929.0}&692.1&\textbf{0.496}&5082.9&682.5&0.500&\textbf{2153.5}& -19.0& 0.502\\
  \multirow{1}{*}{OR (NAIVE-Y)}  
  & 2921.0&684.5& 0.500&\textbf{5127.9}&\textbf{691.6}&\textbf{0.496}&\textbf{2206.8}&\textbf{14.1}&\textbf{0.497}\\
  \multirow{1}{*}{OR (OURS)}
  & 2924.8&\textbf{693.6}& 0.501&5104.1&688.4& 0.498&2179.3&-10.4&0.498\\
  \multirow{1}{*}{IPW (NAIVE-S)}
  & \textbf{2929.8}&692.8&\textbf{0.500}&5101.3&692.1&0.496&2171.5& -1.4&0.498\\
  \multirow{1}{*}{IPW (NAIVE-Y)}  
  &2914.3&677.8& 0.506&\textbf{5127.9}& 689.0&0.497&\textbf{2215.2}&\textbf{22.5}&\textbf{0.494} \\
  \multirow{1}{*}{IPW (OURS)}
  & 2927.1&\textbf{695.9}&0.502&5116.4&\textbf{695.7}&\textbf{0.497}&2189.3& -0.3& 0.497\\
  \multirow{1}{*}{Pro. (NAIVE-S)}
  & 2965.4& 718.8&\textbf{0.485}&5176.6& 740.8&0.490&2245.8& 73.0&0.490\\
  \multirow{1}{*}{Pro. (NAIVE-Y)}  
  &2938.8& 702.5& 0.490&5171.8&722.3&0.490&\textbf{2294.1}&\textbf{99.7}&\textbf{0.487} \\
  \multirow{1}{*}{Pro. (OURS)}
  &\textbf{2968.8}&\textbf{734.8}&0.488&\textbf{5183.4}&\textbf{742.7}&\textbf{0.490}&2253.8& 52.4&0.492 \\
  \bottomrule
  \end{tabular}}%
  \end{sc}
\label{tab:fixed_missing}%
\end{table*}%

\subsection{Experimental Setup}

\paragraph{Datasets.} 
We perform extensive experiments on three widely used benchmark datasets, 
\textsc{IHDP}~\cite{hill2011bayesian}, \textsc{JOBS} \cite{lalonde1986evaluating}, and PRODUCT~\cite{gao2022kuairec}. The \textsc{IHDP} dataset investigates the effects of high-quality home visits on the children's future cognitive scores. It consists of 747 units (139 treated, 608 controlled) and 25 features that measure the characteristics of the children and their mothers. 
Note that we observe only one outcome from one treatment for each unit, and both datasets do not collect the long-term effects. Thus, following previous generation mechanisms~\cite{cheng2021-WSDM, li2023trustworthy}, we simulate the potential short-term outcomes as follows: 
\begin{align} \label{eq:idhp_s}
\begin{split}
S_i(0) &\sim \mathrm{Bern}(\sigma(w_0X_i+\epsilon_{0, i})),\\
S_i(1) & \sim  \mathrm{Bern}(\sigma(w_1X_i+\epsilon_{1, i})),\\
\end{split}
\end{align}
where $\sigma(\cdot)$ is the sigmoid function, $w_0 \sim \mathcal{N}_{[-1,1]}(0, 1)$ follows a truncated normal distribution, $w_1 \sim \mathrm{Unif}(-1, 1)$ follows a uniform distribution, $\epsilon_{0, i} \sim \mathcal{N}(\mu_0, \sigma_0)$, and $\epsilon_{1, i} \sim \mathcal{N}(\mu_1, \sigma_1)$. 
We set $\mu_0=1, \mu_1=3$ and $ \sigma_0=\sigma_1=1$ for \textsc{IHDP} dataset. 
Regarding generating long-term outcomes $Y_{i}(0)$ and $Y_{i}(1)$, we introduce the time step $t$: we set the initial value at time step $0$ as $Y_{0, i}(0) = S_i(0), Y_{0, i}(1) = S_i(1)$, then 
generate $Y_{t, i}(0), Y_{t, i}(1)$ following Eq.\eqref{eq:idhp_y}, and we eventually regard the outcome at the last time step $T$ as the long-term reward, $Y_{i}(0)=Y_{T, i}(0), Y_{i}(1) = Y_{T, i}(1)$.
\begin{align} \label{eq:idhp_y}
\begin{split}
Y_{t,i}(0) & \sim  \mathcal{N}(\beta_0X_i, 1) + C \sum\nolimits_{j=0}^{t-1} Y_{j,i}(0) ,\\
Y_{t,i}(1) & \sim  \mathcal{N}(\beta_1X_i + 2, 0.5) + C \sum\nolimits_{j=0}^{t-1} Y_{j,i}(1),\\
\end{split}
\end{align}
where $\beta_0$ is randomly sampled from $\{0, 1, 2, 3, 4\}$ with probabilities $\{0.5,0.2,0.15,0.1,0.05\}$, $\beta_1 \sim 4\cdot \mathcal{N}_{[0,4]}(0, 1)$, and $C=0.02$ is a scaling factor.

The second dataset, \textsc{JOBS}, explores the effects of job training on income and employment status. It consists of 2,570 units (237 treated, 2,333 controlled), with 17 covariates from observational studies. We employ Eq.\eqref{eq:idhp_s} to simulate short-term outcomes with $\mu_0=0, \mu_1=2$ and $\sigma_0=\sigma_1=1$. We generate long-term outcomes in the similar way as \textsc{IHDP} with the following generation mechanism,
\begin{align} \label{eq:jobs_y}
\begin{split}
Y_{t, i}(0) & \sim  \mathrm{Bern}(\sigma(\beta_0X_i)+C \sum\nolimits_{j=0}^{t-1} Y_{j,i}(0)) + \epsilon_{0, i},\\
Y_{t, i}(1) & \sim  \mathrm{Bern}(\sigma(\beta_1X_i)+C \sum\nolimits_{j=0}^{t-1} Y_{j,i}(0)) + \epsilon_{1, i},\\
\end{split}
\end{align}
where for $\epsilon_{0, i}$ and $\epsilon_{0, i}$, we set $\mu_0=\mu_1=0, \sigma_0=1$ and $\sigma_1=0.5$, and $C=0.02/t$. 
Eventually, we adopt the mechanism below to generate the missing indicator $R$: we first calculate the value of $\text{score}_i = S_i + \sum_{j=1}^d X_{ij}$ for each unit $i$, where $X_{ij}$ is the $j$-th element of $X_i$, and $d$ is the dimension of $X_i$. Then, we specify a missing ratio $\gamma$ and select the $\gamma N$ units with the highest $\text{score}_i$ to be missing.  


The third dataset PRODUCT is collected from a short video-sharing platform, and it is an almost fully observed industrial dataset. There are 4,676,570 samples from 1,411 users on 3,327 items with a density of 99.6\%. We choose video-watching ratios greater than two as 1, otherwise as 0, being short-term outcomes, and generate the long-term outcomes $Y$ the same as Eq. (\ref{eq:jobs_y}). 
The dataset is available at \url{https://github.com/chongminggao/KuaiRec}.

\subsection{Experimental Results}

\noindent\textbf{Experimental details.} 
We aim to learn the optimal policy based on the efficient estimators of long-term reward $\hat \V(\pi; y)$ and short-term reward $\hat \V(\pi; s)$. 
For ease of comparison, we transform the optimization problem into $\arg \max_{\pi \in \Pi}   (1-\lambda)\hat \V(\pi; s)  + \lambda  \hat \V(\pi; y)$, where $\lambda$ is a balance factor between short and long-term rewards. Note that this transformation would not influence the theoretical results shown in Sections \ref{sec-5}-\ref{sec6}. 
Subsequently, 
we include three baselines (DM, IPW, and OR) for comparison.
%
Based on Lemma \ref{lemma:optimal_policy}, the DM (direct method) estimates the optimal policy with
$\hat \pi^*(x) = \mathbb{I}(\hat \tau_s(x) + \lambda \hat \tau_y(x) \geq 0)$, where $\hat \tau_s(x) =\hat \mu_1(x) - \hat \mu_0(x)$ and $\hat \tau_y(x) =\hat m_1(x) - \hat m_0(x)$.
%
%
The IPW (inverse probability weighting) estimators of the short-term and long-term rewards are given as, 
\begin{gather*}
\begin{align*}
\mathbb{\hat V}(\pi; s)^{IPW}={}&\frac{1}{n}\sum_{i=1}^n \Big [\frac{\pi(X_i)A_i S_i }{\hat e(X_i)} + \frac{(1-\pi(X_i)) (1-A_i)S_i }{1 - \hat e(X_i)}\Big ], \\
\mathbb{\hat V}(\pi; y)^{IPW} ={}& \frac{1}{n} \sum_{i=1}^n \Big [ \frac{ \pi(X_i) A_i R_i Y_i }{\hat e(X_i) \hat r(1, X_i, S_i)} + \frac{( -\pi(X_i))(1-A_i)R_iY_i}{(1-\hat e(X_i)) \hat r(0,X_i,S_i)}\Big ].
 \end{align*} 
\end{gather*}
The OR (outcome regression) estimators are given as 
\begin{gather*}
\begin{align*}
\mathbb{\hat V}(\pi; s)^{OR}={}& \frac{1}{n} \sum_{i=1}^n [ \pi(X_i) \hat \mu_1(X_i) + (1- \pi(X_i)) \hat \mu_0(X_i) ], \\
 \mathbb{\hat V}(\pi; y)^{OR} ={}& \frac{1}{n} \sum_{i=1}^n [ \pi(X_i) \hat{\tilde m}_1(X_i, S_i) + (1- \pi(X_i)) \hat{\tilde m}_0(X_i, S_i) ].
 \end{align*} 
\end{gather*}
For each new baseline, we compare three different optimization strategies: $\textsc{Naive-S}$ ($\lambda = 0$), $\textsc{Naive-Y}$ ($\lambda = 1$), and Ours (Balanced, $\lambda = 0.5$).


We report the rewards, the changes in welfare, and policy errors with different balance factors. 
Formally, the short-term reward of the learned policy $\hat \pi (X)$ is $\hat \V(\pi; s) = \sum \nolimits_{i=1}^{n} [\hat \pi(X) S(1) + (1-\hat \pi(X))S(0)]$ with $\lambda = 0$, the long-term is $\hat \V(\pi; y) = \sum \nolimits_{i=1}^{n} [\hat \pi(X) Y(1) + (1-\hat \pi(X))Y(0)]$ with $\lambda = 1$, and the balanced reward is $0.5 \hat \V(\pi; s) + 0.5\hat \V(\pi; Y)$ with $\lambda = 0.5$. 
Similar as~\citet{2018Who,li2023trustworthy}, the welfare changes are defined as $\Delta W_s = \sum \nolimits_{i=1}^{n}\left[(S_i(1) - S_i(0))\cdot \hat \pi(X_i)\right]$ for short-term-based ($\lambda=0$),  $\Delta W_y = \sum \nolimits_{i=1}^{n}\left[(Y_i(1) - Y_i(0))\cdot \hat \pi(X_i)\right]$ for long-term-based ($\lambda=1$), and  $0.5\Delta W_s + 0.5\Delta W_y$ for the overall balanced-based rewards ($\lambda=0.5$).   
The policy error is defined as  $1/n \sum \nolimits_{i=1}^{n} || \pi^*_0 (X_i) - \hat \pi (X_i)||^2$, which is the mean square errors between the estimated policy $\hat \pi(X)$ and the optimal policy $\pi_0^*(X)$ in Lemma   \ref{lemma:optimal_policy}. The value of $\pi^*_0(X_i)$ are derived with different $\lambda$ as well. Among these evaluation metrics, 
BALANCE REWARD is the most critical here, as it directly 
underscores the need for a harmonious trade-off between immediate gains and sustained benefits. 

\begin{figure}[t!]
\centering
\subfigure[Correlated case on \textsc{IHDP}]{
\includegraphics[width=0.23\textwidth]{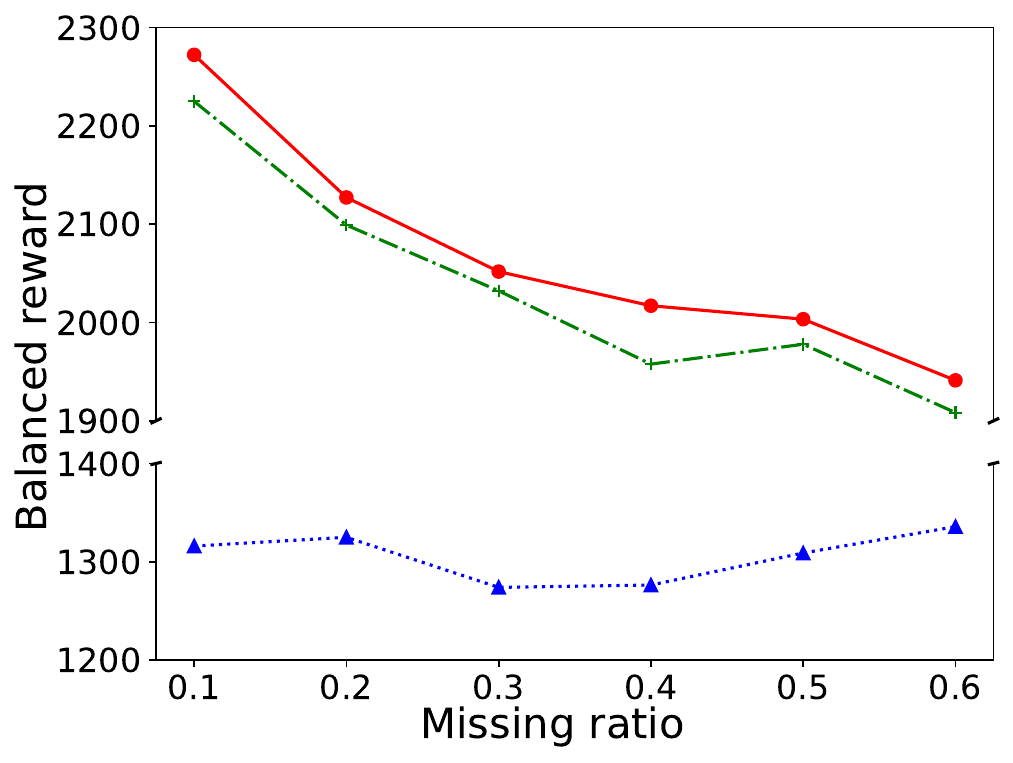}} 
\subfigure[Correlated case on \textsc{JOBS}]{
\includegraphics[width=0.23\textwidth]{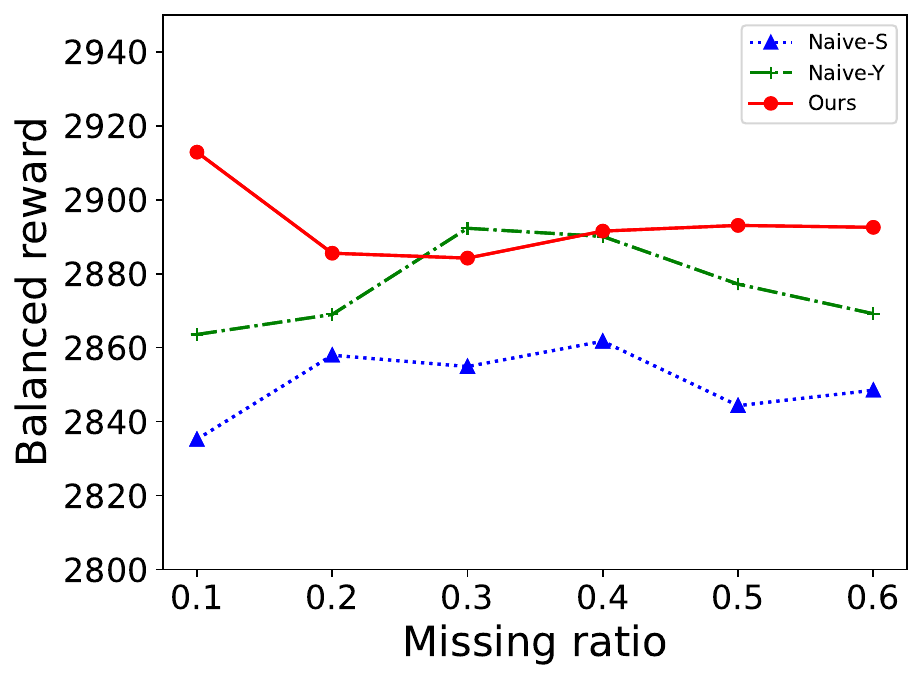}}
\subfigure[Uncorr. case on \textsc{IHDP}]{
\includegraphics[width=0.23\textwidth]{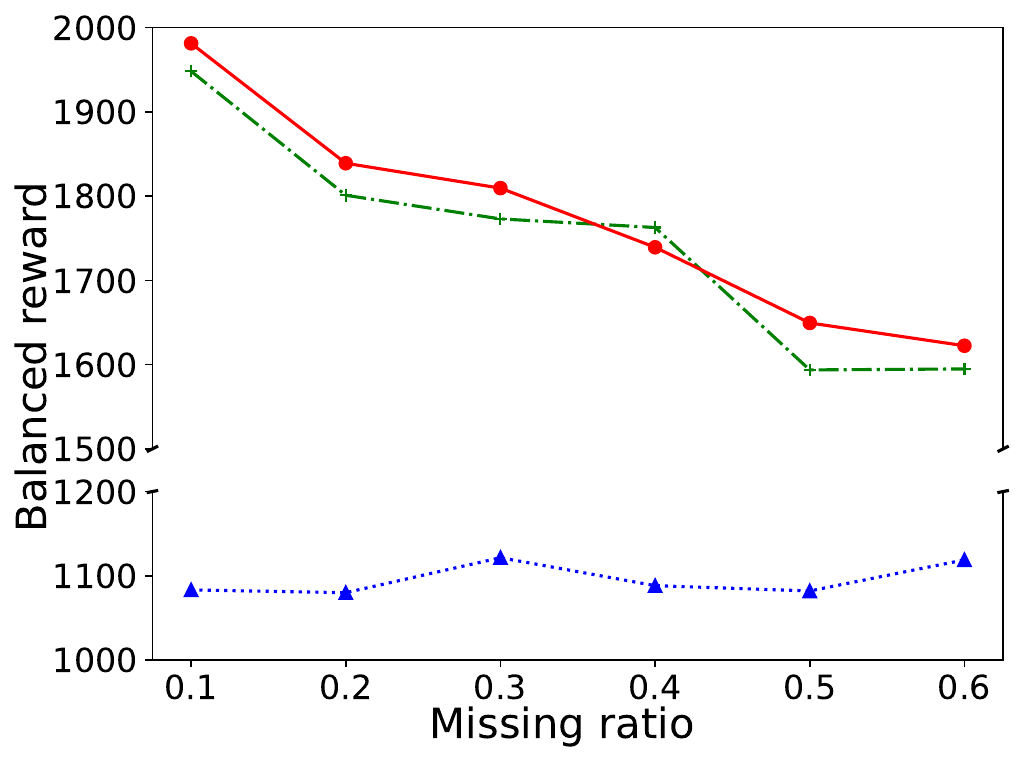}}
\subfigure[Uncorr. case on \textsc{JOBS}]{
\includegraphics[width=0.23\textwidth]{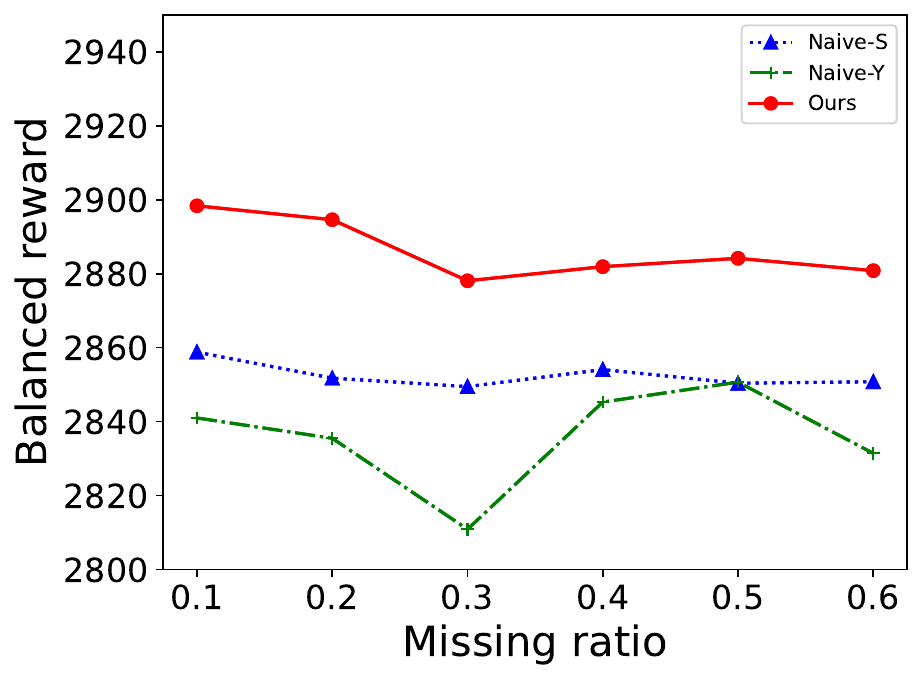}}
\caption{Comparison of \textsc{Naive-S}, \textsc{Naive-Y} and our method with different missing ratios of $Y$ on \textsc{IHDP} and \textsc{JOBS}, where the metric is the balanced reward.}
\label{fig:diff_miss}
\end{figure}

\noindent\textbf{Policy learning with short-term and long-term reward.} 
We average over 50 independent trials of policy learning with short-term and long-term rewards in \textsc{JOBS} and \textsc{PRODUCT}, and the results are shown in Table~\ref{tab:fixed_missing}. The bold fonts represent the best results among the 3 strategies ($\textsc{Naive-S}, \textsc{Naive-Y}, \textsc{Ours}$) for each baseline ($\textsc{IPW, OR, DM}$, and $ \textsc{Pro.}$). We fix the missing ratio of outcomes $Y$ to be $\textsc{0.1}$ and the number of time steps is $T=10$.
As for our proposed methods, i.e., Proposed (Naive-S), Proposed (Naive-Y), and Proposed (Ours) in the last three lines of the table, we see that Proposed (Ours) obtains the highest balanced reward. Besides, balanced rewards for all methods are higher than the short-term and long-term rewards. This demonstrates the necessity of balancing short-term and long-term rewards.
Regarding the various estimators of $\mathbb{V}(\pi; y)$ and $\mathbb{V}(\pi; s)$, including the Proposed, IPW, OR, and DM estimators, we observe that our proposed method outperforms the other methods in terms of balanced reward metrics. This demonstrates the superiority of the proposed method.
 More results are given in Appendix~\ref{app:more_exp}. 

\noindent\textbf{Effects of varying missing ratios.} 
We study the effects of varying missing ratios for long-term outcome $Y$. As shown in Figures~\ref{fig:diff_miss}(a) and \ref{fig:diff_miss}(b), our method achieves better performance in almost all scenarios. 
As the missing ratio increases, both \textsc{Naive-Y} and our method exhibit a declining trend in performance. 
This decline is expected, as higher missing ratios mean more long-term outcomes are neglected. The performance of $\textsc{Naive-Y}$ is consistently worst.  

\noindent\textbf{Effects of varying correlation between $S$ and $Y$.} 
Data generation mechanisms for $Y$ in Eqs. \eqref{eq:idhp_y} and \eqref{eq:jobs_y} inherently lead to $S \notindep Y | X$. 
To compare the distinction between cases with varying correlations, we also generate $Y$ that satisfies $S \indep Y | X$, the data generation details are provided in Appendix~\ref{app:uncorr_data}. The results are shown in Figures \ref{fig:diff_miss}(c) and \ref{fig:diff_miss}(d). 
Importantly, comparing Figure \ref{fig:diff_miss}(a) with Figure \ref{fig:diff_miss}(c), and Figure \ref{fig:diff_miss}(b) with Figure \ref{fig:diff_miss}(d), respectively, we observe that the performance of correlated cases surpasses that in uncorrelated cases. This empirical observation aligns with our findings in Proposition~\ref{prop-efficiency}.

\noindent\textbf{Effects of varying time steps.} 
We further study the impact of varying time steps on long-term outcomes, the associated results are displayed in Figures \ref{fig:diff_time}(a) and \ref{fig:diff_time}(b), where the missing ratio is set as 0.6. Overall, our method consistently outperforms other baselines across all time steps, even in scenarios with a high missing ratio. 
 More numerical results are available in Appendix~\ref{app:more_exp}. 
  



\noindent\textbf{Effects of varying costs.} 
According to Eq.\eqref{eq:cost}, we explore the effects of different costs. As depicted in Figures \ref{fig:diff_time}(c) and \ref{fig:diff_time}(d), in all scenarios with various costs, our method achieves higher balanced rewards compared with \textsc{Naive-S} and \textsc{Naive-Y}, which empirically demonstrates the superiority of taking long-and-short-term rewards into account.  


\begin{figure}[t!]
\centering
\subfigure[Varying time steps on \textsc{IHDP}]{
\includegraphics[width=0.23\textwidth]{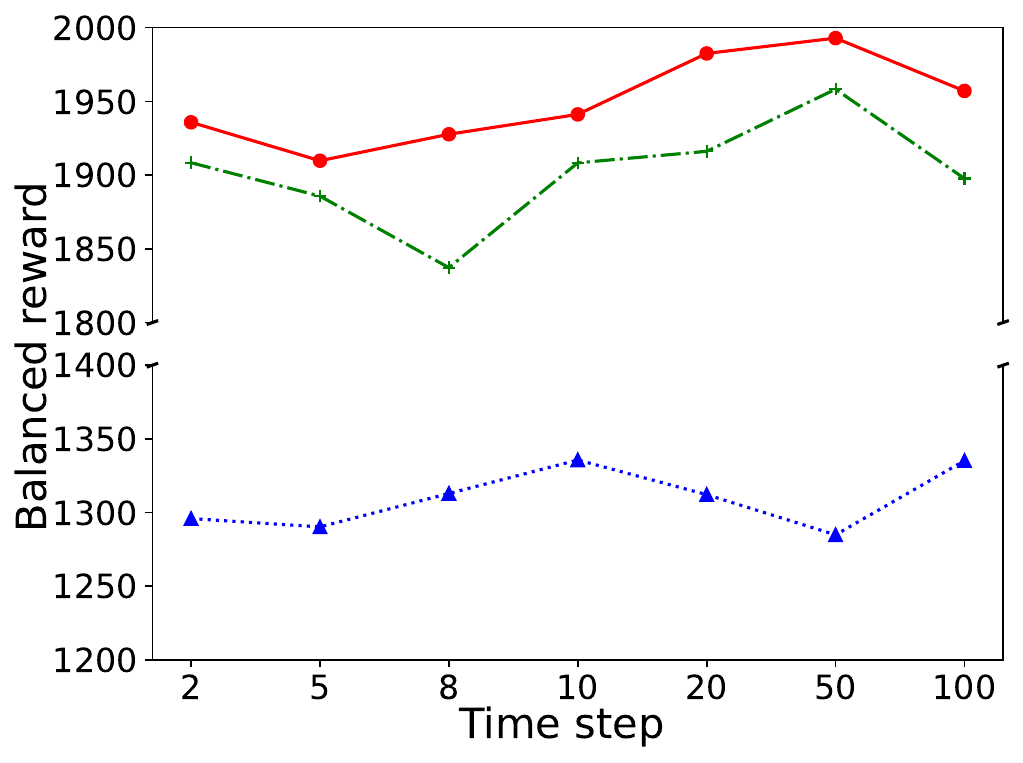}} 
\subfigure[Varying time steps on \textsc{JOBS}]{
\includegraphics[width=0.23\textwidth]{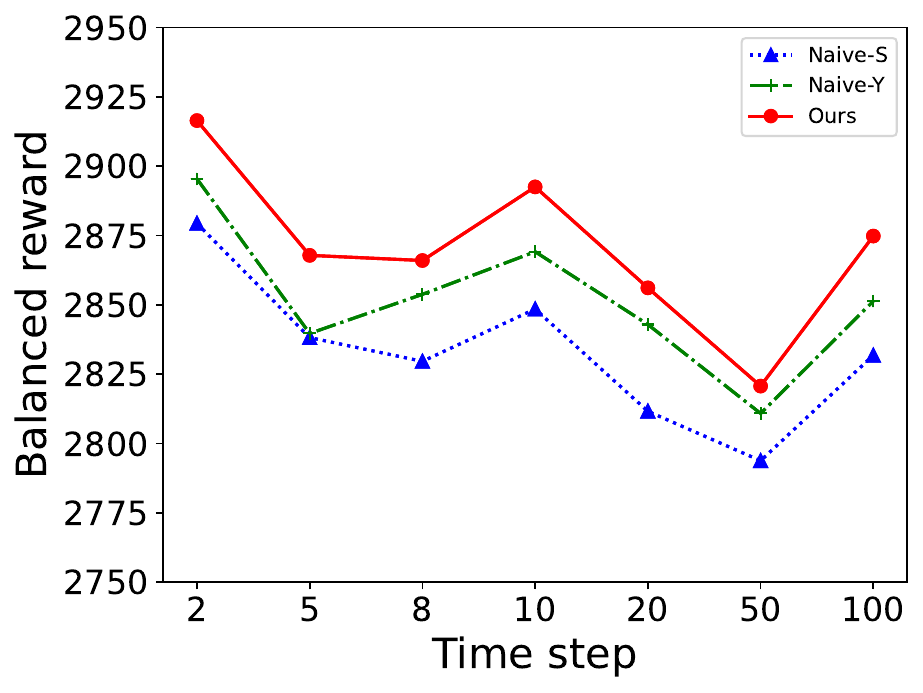}}
\subfigure[Varying costs on \textsc{IHDP}]{
\includegraphics[width=0.23\textwidth]{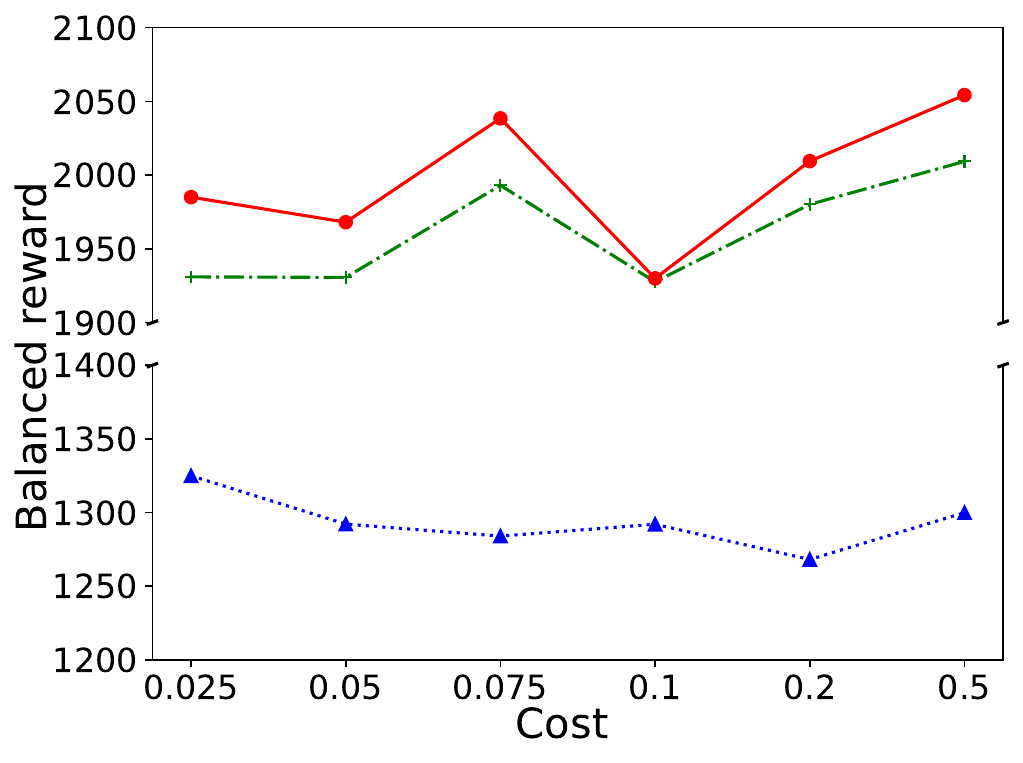}} 
\subfigure[Varying costs on \textsc{JOBS}]{
\includegraphics[width=0.23\textwidth]{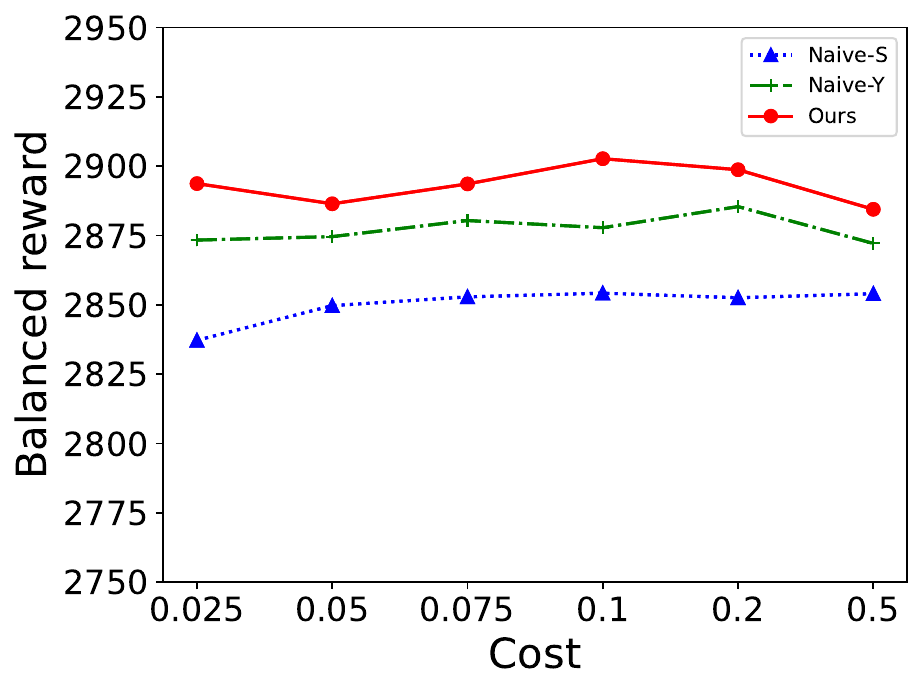}}
\caption{Comparison of \textsc{Naive-S}, \textsc{Naive-Y} and ours on \textsc{IHDP} and \textsc{JOBS}, where the metric is the balanced reward.}
\label{fig:diff_time}
\end{figure}


\section{Conclusion}
This study delves into an important aspect of interest to empirical researchers and decision-makers in many fields -- balancing short-term and long-term rewards. We propose a principled policy learning approach for achieving this goal, which consists of two key steps: estimating the short/long-term rewards for a given policy and learning the optimal policy by taking the estimated short/long-term rewards as the objective functions. 
We conduct a comprehensive theoretical analysis and perform extensive experiments to demonstrate the effectiveness of the proposed policy learning approach.  
A limitation of this work is that Assumption \ref{assump5-1} does not hold in the presence of unmeasured confounders that affects both treatment, short-term and long-term outcomes. Future efforts should focus on extending our method and theory by relaxing identifiability assumptions. 


\section*{Acknowledgements}
This research was supported by the National Natural Science Foundation of China (No. 12301370) and the disciplinary funding of Beijing Technology and Business University (STKY202301). 
Feng Xie is supported by the Natural Science Foundation of China (No. 62306019).

\section*{Impact Statement}
The paper introduces a novel policy learning approach designed to effectively balance short-term and long-term rewards, overcoming challenges such as confounding bias and missing data in long-term outcomes. This research provides valuable insights and practical implications, particularly in scenarios where optimizing both short-term and long-term outcomes is crucial. Here are some potential applications: 

(a) Marketing and customer behavior: marketing professionals can optimize incentive strategies, ensuring they influence customer behavior positively in both the short and long term; 
(b) Information technology (IT) and user experience: IT companies can design web pages that not only cater to immediate user preferences but also enhance user engagement and satisfaction over an extended period; (c) Healthcare and treatment strategies: medical practitioners can refine drug prescriptions, considering both short-term alleviation and long-term outcomes in chronic diseases like Alzheimer's and AIDS; (d) Labor market and employment programs: policymakers can enhance the design of job training programs, considering both immediate income impacts and subsequent improvements in employment status; (e) Video recommendation and content engagement: content providers can optimize recommendations, avoiding short-term clickbait strategies that may lead to user churn, ensuring sustained user engagement and revenue growth; etc. 


\bibliography{reference}
\bibliographystyle{icml2024}

\newpage
\appendix
\onecolumn


\def\thefigure{\arabic{figure}}
\def\thetable{\arabic{table}}

 \renewcommand{\theequation}{A.\arabic{equation}}
 \renewcommand{\thetable}{A\arabic{table}}
 \renewcommand{\thetheorem}{A\arabic{theorem}}
 \renewcommand{\theproposition}{A\arabic{proposition}}
 \setcounter{equation}{0}

\section{Proofs of Proposition \ref{prop5-3} and Theorem \ref{thm-EIF}}  \label{app-a}

{\bf Proposition \ref{prop5-3}} (Identifiability of $\V(\pi; y)$). 
\emph{Under Assumptions \ref{assump5-1}-\ref{assump5-2}, the long-term reward $\V(\pi; y)$ is identified as 
 	\begin{align*}
		     \V(\pi; y) 
		     		={}&  \E[ \pi(X)  \tilde m_1(X, S)  + (1 - \pi(X))  \tilde m_0(X, S) ],
		 \end{align*}    
where $\tilde m_a(X, S) = \E[Y | X, S, A=a, R=1 ]$ for $a = 0, 1$. }  		 

\begin{proof}[Proof of Proposition \ref{prop5-3}] The identifiability of  $\V(\pi; y)$ can be obtained by noting that 
	 	\begin{align*}
		     \V(\pi; y) ={}&  \E[ \pi(X) Y(1) + (1 - \pi(X)) Y(0)] \\
       ={}& \E[ \pi(X)\cdot \E(Y(1)| X, S(1)) + (1 - \pi(X)) \cdot \E(Y(0)| X, S(0)) ]  \\
		     	={}& \E[ \pi(X) \cdot \E(Y(1)| X, S(1), A=1)  + (1 - \pi(X)) \cdot \E(Y(0)| X, S(0), A=0) ]  \\
				={}& 	\E[ \pi(X) \cdot \E(Y(1)| X, S(1), A=1, R=1) ] + \E[  (1 - \pi(X)) \cdot \E(Y(0)| X, S(0), A=0, R=1) ] \\
				 ={}& 	\E[ \pi(X)  \tilde m_1(X, S) + (1 - \pi(X))  \tilde m_0(X, S) ],
		 \end{align*}    
where the second equality follows by the law of iterated expectations, the third equality follows from Assumption \ref{assump5-1}, and the fourth equality follows from Assumption \ref{assump5-2}. 
\end{proof}

{\bf Theorem \ref{thm-EIF}} (Efficiency Bounds of $\V(\pi; s)$ and $\V(\pi; y)$).
\emph{Under Assumptions \ref{assump5-1} and  \ref{assump5-2},  we have that} 

\emph{(a) the efficient influence function of $\V(\pi; s)$ is $\phi_s - \V(\pi; s)$, where  
	\begin{align*}
	        \phi_s &={}   \{ \pi(X) \mu_1(X) + (1 - \pi(X)) \mu_0(X)\} \\ 
	         +{}& \frac{\pi(X) A(S - \mu_1(X))}{e(X)}  +  \frac{(1 - \pi(X)) (1-A)(S - \mu_0(X))}{1 - e(X)},
	\end{align*}
the associated semiparametric efficiency bound is $\text{Var}(\phi_s)$.}

\emph{(b) the efficient influence functions of $\V(\pi; y)$ is  $\phi_y - \V(\pi; y)$, where  
	\begin{align*}
	\phi_{y} &={} \{ \pi(X) m_1(X) + (1 - \pi(X)) m_0(X) \} \\
	  +{}&  \frac{\pi(X) A R (Y - \tilde m_1(X, S))}{e(X) r(1, X, S)}  + \frac{\pi(X) A  ( \tilde  m_1(X, S) - m_1(X))}{e(X)}   \\
	  +{}&   \frac{(1 - \pi(X))(1-A) R (Y - \tilde  m_0(X, S))}{(1-e(X)) r(0, X, S)}   \\
	    +{}&  \frac{(1 - \pi(X))(1-A) ( \tilde m_0(X, S) - m_0(X)) }{1 - e(X)},
	\end{align*}
the associated semiparametric efficiency bound is $\text{Var}(\phi_y)$. }

\begin{proof}[Proof of Theorem \ref{thm-EIF}]
    
Let $f(\cdot)$ be the probability density/mass function, $f_1(y| x, s)$ and  $f_0(y| x, s)$ be the density of $Y(1)$ and $Y(0)$ conditional on  $(X=x,  S(1) = s)$ and $(X=x,  S(0) = s)$ respectively, and denote $f_1(s | x)$ and  $f_0(s|x)$ be the density of $S(1)$ and $S(0)$ conditional on  $X=x$ respectively. Then the observed data distribution under Assumptions \ref{assump5-1} and \ref{assump5-2}  is given as 
\begin{align*} 
 & p(a, x, s, y, r)  \\
  ={}&   f(a, x, s, y, r = 1)^r  \times  f(a, x, s, r = 0)^{1-r}    \\
   ={}&   [  f(r = 1| a, x, s, y) f(a, x, s, y) ]^r \times [ f(r = 0| a, x, s) f(a, x, s) ]^{1-r}  \\ 
   ={}&   [  r(a, x, s) \cdot \{ f(s, y| a=1, x) f(x) e(x) \}^a \cdot \{ f(s, y| a=0, x) f(x) (1-e(x)) \}^{1-a}  ]^r  \\
        & \times [  (1- r(a,x,s)) \cdot \{ f(s| x, a = 1) f(x) e(x)  \}^a \cdot \{ f(s| x, a = 0) f(x) (1-e(x))  \}^{1-a}   ]^{1-r}  \\ 
   ={}&  f(x) \times  [  r(a, x, s) \cdot \{ f_1(y| x, s) f_1(s| x) e(x) \}^a \cdot \{ f_0(y|  x, s) f_0(s| x)  (1-e(x)) \}^{1-a}  ]^r  \\
        & \times [  (1- r(a,x,s)) \cdot \{ f_1(s| x) e(x)  \}^a \cdot \{ f_0(s| x)  (1-e(x))  \}^{1-a}   ]^{1-r}. 
\end{align*}

Under Assumptions \ref{assump5-1} and \ref{assump5-2}, consider a regular parametric submodel indexed by $\theta$ given as 
 	\begin{align*}
	  p(a, x, s, y, r;  \theta) ={}& f(x, \theta) \times  [  r(a, x, s, \theta) \cdot \{ f_1(y| x, s, \theta) f_1(s| x, \theta) e(x, \theta) \}^a \cdot \{ f_0(y| x, s, \theta) f_0(s| x, \theta)  (1-e(x, \theta)) \}^{1-a}  ]^r  \\
        & \times [  (1- r(a,x,s, \theta)) \cdot \{ f_1(s| x, \theta) e(x, \theta)  \}^a \cdot \{ f_0(s| x, \theta)  (1-e(x, \theta))  \}^{1-a}   ]^{1-r}. 
\end{align*}
which equals $p(a, x, y, g)$ when $\theta = \theta_0$. Also, $f_a(y | x, s,  \theta_0) = f_a(y | a, x, s, \theta_0)  = f_a(y | a, x, s, r=1, \theta_0)$ for $a = 0, 1$ by Assumptions \ref{assump5-1} and \ref{assump5-2}. 

 Then, the score function for this submodel is given by 
	 \begin{align*}
	 	& s(a, x, s, y, r; \theta) 
		 = \frac{\partial \log p(a, x, s, y, r;  \theta)}{\partial \theta}  \\
			={}&  s(x, \theta) + r a \cdot \{ s_1(y|x, s, \theta) + s_1(s | x, \theta) \} +  r(1-a) \cdot \{ s_0(y| x, s, \theta) + s_0(s | x, \theta) \} \\
			 &{} + (1-r) a \cdot s_1(s | x, \theta)  +  (1-r)(1-a) \cdot s_0(s | x, \theta) \\
			 &{} +    \frac{a - e(x,\theta)}{ e(x,\theta)(1 - e(x,\theta)) } \dot{e}(x, \theta) +  \frac{r - r(a, x, s,\theta)}{ r(a, x, s,\theta)(1 - r(a, x, s,\theta)) } \dot{r}(a, x, s,\theta), \\
			 	={}&  s(x, \theta) + r a \cdot  s_1(y| x, s, \theta)  +  r(1-a) \cdot s_0(y|  x, s, \theta)  \\
			 &{} +  a \cdot s_1(s | x, \theta)  +  (1-a) \cdot s_0(s | x, \theta) \\
			 &{} +    \frac{a - e(x,\theta)}{ e(x,\theta)(1 - e(x,\theta)) } \dot{e}(x, \theta) +  \frac{r - r(a, x, s,\theta)}{ r(a, x, s,\theta)(1 - r(a, x, s,\theta)) } \dot{r}(a, x, s,\theta),    
	 \end{align*}	
where 
	\begin{align*}
		s(x, \theta)  ={}& \frac{\partial \log f(x, \theta) }{ \partial \theta}, \\
		s_1(y| x, s, \theta)  ={}& \frac{\partial \log f_1(y| x, s, \theta)  }{ \partial \theta}, \\
		s_0(y| x, s, \theta)  ={}& \frac{\partial \log f_0(y|  x, s, \theta)  }{ \partial \theta}, \\  	
		s_1(s | x, \theta)  ={}& \frac{\partial \log f_1(s | x, \theta)  }{ \partial \theta}, \\  
			s_0(s | x, \theta)  ={}& \frac{\partial \log f_0(s | x, \theta)  }{ \partial \theta}, \\  
			\dot{e}(x,\theta) ={}& \frac{\partial e(x,\theta)}{\partial \theta}, \\
		\dot{r}(a, x, s, \theta) ={}& \frac{ \partial r(a, x, s, \theta)}{ \partial \theta}. 				
	\end{align*}

Thus, the tangent space $\mathcal{T}$ is 		
 \begin{align*}    \mathcal{T} ={}&  \{ s(x) + r a  s_1(y| x, s) + r(1-a)  s_0(y|  x, s)  + a s_1(s|x) + (1-a) s_0(s|x) \\
      {}& +   (a - e(x)) \cdot b(x)  + (r - r(a, x, s)) \cdot c(x)     \},  
 \end{align*} 
where $s(x)$ satisfies $\E[s(X)] = \int s(x) f(x)dx = 0$, $s_a(y|a, x, s)$ satisfies $\E[ s_a(Y| X, S) \big | X=x, S=s ] =  \int s_a(y| x, s) f_a(y| x, s) dy = 0$ for $a = 0,1$, $s_a(s | x)$ satisfies $\E[ s_a(S | X) \big | X=x ] =  \int s_a(s | x) f_a(s| x) ds = 0$ for $a = 0,1$, and  $b(x)$ and $c(x)$ are arbitrary square-intergrable measurable functions of $x$.  
In addition,  $s_a(y|a, x, s) = s_a(y| a, x, s, r = 1)$ according to Assumptions \ref{assump5-1} and \ref{assump5-2} (i.e., $f_a(y |  x, s) = f_a(y | a, x, s, r=1)$).

 \medskip \noindent 
{\bf Efficient influence function of short-term reward.}    Under the above parametric submodel, the short-term reward $\V(\pi; s)$ can be written as 
	\begin{align*}
	 \V(\pi, &\theta; s) ={} \E[ \pi(X) S(1) + (1 - \pi(X)) S(0) ]        \\
	 		={}&  \int\int \pi(x) s f_1(s |x, \theta) f(x, \theta) ds dx +  \int\int (1-\pi(x)) s f_0(s |x, \theta) f(x, \theta) ds dx.  
	\end{align*} 
The pathwise derivative of $\V(\pi, \theta; s)$ at $\theta = \theta_0$ is given as  
	\begin{align*}
	 \frac{\partial \V(\pi, \theta; s)}{\partial \theta}&\Big|_{\theta = \theta_0} ={} \int \int \pi(x) s \cdot s_1(s| x, \theta_0) f_1(s|x,\theta_0) f(x, \theta_0) ds dx  +  \int \int \pi(x) s \cdot f_1(s|x,\theta_0) s(x, \theta_0) f(x, \theta_0)  ds dx  \\
	 	{}&+  \int \int (1-\pi(x)) s \cdot s_0(s| x, \theta_0) f_0(s|x,\theta_0) f(x, \theta_0) ds dx  +  \int \int (1-\pi(x)) s \cdot f_0(s|x,\theta_0) s(x, \theta_0) f(x, \theta_0)  ds dx  \\
		={}& \E\Big [ \pi(X) \cdot  \E\Big \{ S(1) \cdot s_1(S(1)|X) \Big | X\Big \} \Big ] +   \E\Big [ (1-\pi(X)) \cdot  \E\Big \{ S(0) \cdot s_0(S(0)|X) \Big | X\Big \} \Big ]  \\
		{}&+ \E\Big [ s(X) \Big  \{  \pi(X) \mu_1(X) + (1 - \pi(X)) \mu_0(X) \Big \}  \Big ]. 
	\end{align*} 	
Next, we construct the efficient influence function of $\V(\pi; s)$. Let 
	\begin{align*}
	\tilde \phi_{s} = \pi(X) \frac{A(S - \mu_1(X))}{e(X)}  + (1 - \pi(X)) \frac{(1-A)(S - \mu_1(X))}{1 - e(X)} + \{ \pi(X) \mu_1(X) + (1 - \pi(X)) \mu_0(X) - \V(\pi; s) \}. 
	\end{align*}
Pathwise differentiability of $\V(\pi; s)$ can be verified by 
	\begin{equation}  \label{A.1}
	 \frac{\partial \V(\pi, \theta; s)}{\partial \theta} \Big|_{\theta = \theta_0}  =  \E[ \tilde  \phi_{s} \cdot s(A, X, S, Y, R; \theta_0) ], 
	\end{equation}
which implies that $\tilde  \phi_{s}$ is an influence function of $\V(\pi; s)$.  Now we give a detailed proof of (\ref{A.1}). 
	\begin{align*} 
 \E[ \tilde  \phi_{s} \cdot s(A, X, S, Y, R; \theta_0) ]
	={}&  H_1 +  H_2 + H_3,
	\end{align*}
 where 
 	\begin{align*}
		H_1 ={}& \E \left [  \pi(X) \frac{A(S - \mu_1(X))}{e(X)}  \cdot s(A, X,  S, Y, R; \theta_0) \right ] \\
			={}&  \E\Biggl [  \pi(X) \frac{A(S - \mu_1(X))}{e(X)}    \\
			 &{}\times \left \{ s(X) + RA \cdot s_1(Y| X, S) + A \cdot s_1(S|X) +  \frac{A - e(X)}{ e(X)(1 - e(X)) } \dot{e}(X) + \frac{R - r(A, X, S)}{r(A, X, S)(1 - r(A, X, S))} \dot{r}(A, X, S) \right \} \Biggr ] \\
			 ={}&  \E \left [  \pi(X) \frac{A(S - \mu_1(X))}{e(X)} \cdot s_1(S| X)  \right ] \\
			={}& \E \left [  \pi(X)  \E \left \{  (S(1) - \mu_1(X)) \cdot  s_1(S(1) |X)     \Big | X \right \}   \right ]   \\
			={}&  \E \left [  \pi(X)  \E \left \{  S(1) \cdot  s_1(S(1) |X)     \Big | X \right \}   \right ]  = \text{ the first term of }  \frac{\partial \V(\pi, \theta; s)}{\partial \theta} \Big|_{\theta = \theta_0}, 
	\end{align*}	
and 	similarly, 
 	\begin{align*}		
		H_2 ={}& \E \left [ (1 - \pi(X)) \frac{(1-A)(S - \mu_1(X))}{1 - e(X)}  \cdot s(A, X,  S, Y, R; \theta_0) \right ] \\
		={}& \E \left [ (1 - \pi(X)) \frac{(1-A)(S - \mu_1(X))}{1 - e(X)}  \cdot s_0(S| X) \right ] \\
	={}&  \E \left [  (1-\pi(X))  \E \left \{  S(0) \cdot  s_0(S(0) |X)     \Big | X \right \}   \right ]  = \text{ the second term of }  \frac{\partial \V(\pi, \theta; s)}{\partial \theta} \Big|_{\theta = \theta_0},
	\end{align*}	
   	\begin{align*}		
		H_3 ={}& \E \left [  \{ \pi(X) \mu_1(X) + (1 - \pi(X)) \mu_0(X) - \V(\pi; s) \} 
 \cdot s(A, X,  S, Y, R; \theta_0) \right ] \\
 ={}& \E \left [  \{ \pi(X) \mu_1(X) + (1 - \pi(X)) \mu_0(X) - \V(\pi; s) \} 
 \cdot s(X) \right ] \\
	={}&   \E\Big [ s(X) \Big  \{  \pi(X) \mu_1(X) + (1 - \pi(X)) \mu_0(X) \Big \}  \Big ]  = \text{ the third term of }  \frac{\partial \V(\pi, \theta; s)}{\partial \theta} \Big|_{\theta = \theta_0},
	\end{align*}

Thus, equation (\ref{A.1}) holds. 
In addition,  let  $s(X) = \{ \pi(X) \mu_1(X) + (1 - \pi(X)) \mu_0(X) - \V(\pi; s) \}$,  $s_1(Y|X) =  \pi(X) \frac{(S - \mu_1(X))}{e(X)} $,   $s_0(S|X) = (1 - \pi(X)) \frac{(S - \mu_1(X))}{1 - e(X)}$,  
  then $\tilde  \phi_{s}$ can be written as 
  	\[  \tilde   \phi_{s} = s(X) +   A s_1(S|X) + (1-A) s_0(S|X).    \]
Clearly, we have that $\int s_a(s |x) f_a(s |x) ds = 0$ and $\int s(x) f(x)dx = 0$, which implies that $\tilde  \phi_{s} \in \mathcal{T}$, and thus $\tilde  \phi_{s}$ is the efficient influence function of $\V(\pi; s)$. 

 \medskip \noindent 
{\bf Efficient influence function of long-term reward.}    Under the above parametric submodel, the long-term reward $\V(\pi; y)$ can be written as 
     \begin{align*}
	 \V(\pi, &\theta; y) ={} \E[ \pi(X) Y(1) + (1 - \pi(X)) Y(0) ]        \\
	 		={}& \E \Big [ \pi(X) \E\{ Y(1) | S(1), X \} + (1 - \pi(X)) \E\{ Y(0) | S(0), X \} \Big ]    \\
	 		={}&  \int\int \int \pi(x) y f_1(y |x, s, \theta) f_1(s|x, \theta) f(x, \theta) dy ds dx +  \int\int \int (1-\pi(x)) y f_0(y|x, s, \theta)  f_0(s |x, \theta) f(x, \theta) dy ds dx.  
	\end{align*}  
	
The pathwise derivative of $\V(\pi, \theta; y)$ at $\theta = \theta_0$ is given as  
	\begin{align*}
	 \frac{\partial \V(\pi, \theta; y)}{\partial \theta}&\Big|_{\theta = \theta_0} ={} \int \int \int \pi(x) y s_1(y|x, s, \theta_0) f_1(y|x, s, \theta_0) \cdot f_1(s|x,\theta_0) f(x, \theta_0) dy ds dx  \\
	 {}&+  \int \int \int \pi(x) y f_1(y|x, s, \theta_0) \cdot \Big \{ s_1(s|x, \theta_0) f_1(s|x,\theta_0) f(x, \theta_0) +   f_1(s|x,\theta_0) s(x, \theta_0)  f(x, \theta_0)   \Big\} dy ds dx  \\
	{}&+  \int \int \int (1-\pi(x)) y s_0(y|x, s, \theta_0) f_0(y|x, s, \theta_0) \cdot f_0(s|x,\theta_0) f(x, \theta_0) dy ds dx  \\
	 {}&+  \int \int \int (1-\pi(x)) y f_0(y|x, s, \theta_0) \cdot \Big \{ s_0(s|x, \theta_0) f_0(s|x,\theta_0) f(x, \theta_0) +   f_0(s|x,\theta_0) s(x, \theta_0)  f(x, \theta_0)   \Big\} dy ds dx  \\
		={}& \E\Big [ \pi(X) \cdot  \E\Big \{ Y(1) \cdot s_1(Y(1)|X, S) \Big | X, S\Big \} \Big ] +   \E\Big [ (1-\pi(X)) \cdot  \E\Big \{ Y(0) \cdot s_0(Y(0)|X, S) \Big | X, S\Big \} \Big ]  \\ 
		{}&+ \E\Big [ \pi(X) \cdot  \E\Big \{ \tilde m_1(X, S) \cdot s_1(S(1)|X) \Big | X\Big \} \Big ] + \E\Big [ (1-\pi(X)) \cdot  \E\Big \{ \tilde m_0(X, S) \cdot s_0(S(0)|X) \Big | X\Big \} \Big ]  \\
		{}&+ \E\Big [ s(X) \Big  \{  \pi(X) \tilde m_1(X, S) + (1 - \pi(X)) \tilde m_0(X, S) \Big \}  \Big ]  \\
			={}& \E\Big [ \pi(X) \cdot  \E\Big \{ Y(1) \cdot s_1(Y(1)|X, S) \Big | X, S\Big \} \Big ] +   \E\Big [ (1-\pi(X)) \cdot  \E\Big \{ Y(0) \cdot s_0(Y(0)|X, S) \Big | X, S\Big \} \Big ]  \\ 
		{}&+ \E\Big [ \pi(X) \cdot  \E\Big \{ \tilde m_1(X, S) \cdot s_1(S(1)|X) \Big | X\Big \} \Big ] + \E\Big [ (1-\pi(X)) \cdot  \E\Big \{ \tilde m_0(X, S) \cdot s_0(S(0)|X) \Big | X\Big \} \Big ]  \\
		{}&+ \E\Big [ s(X) \Big  \{  \pi(X) m_1(X) + (1 - \pi(X)) m_0(X) \Big \}  \Big ],
	\end{align*} 	
where the last equation follows from $\E[ \tilde m_a(X, S)  | X ] = m_a(X)$ for $a = 0, 1$.  
	
Next, we construct the efficient influence function of $\V(\pi; y)$. Let 
	\begin{align*}
	\tilde \phi_{y} ={}& \pi(X) \frac{A R (Y - \tilde m_1(X, S))}{e(X) r(1, X, S)}  + (1 - \pi(X)) \frac{(1-A) R (Y - \tilde m_0(X, S))}{(1-e(X)) r(0, X, S)}    \\
	    {}& + \pi(X) \frac{A  (  \tilde m_1(X, S) - m_1(X))}{e(X)}  + (1 - \pi(X)) \frac{(1-A) (  \tilde m_0(X, S) - m_0(X)) }{1 - e(X)}  \\
	    {}&+ \{ \pi(X) m_1(X) + (1 - \pi(X)) m_0(X) - \V(\pi; y) \}. 
	\end{align*}
Pathwise differentiability of $\V(\pi; y)$ can be verified by 
	\begin{equation}  \label{A.2}
	 \frac{\partial \V(\pi, \theta; y)}{\partial \theta} \Big|_{\theta = \theta_0}  =  \E[\tilde \phi_{y} \cdot s(A, X, S, Y, R; \theta_0) ], 
	\end{equation}
which implies that $\tilde \phi_{y}$ is an influence function of $\V(\pi; y)$.  Now we give a detailed proof of (\ref{A.2}). The right side of (\ref{A.2}) can be decomposed as 
	\[    \E[ \tilde \phi_{y} \cdot s(A, X, S, Y, R; \theta_0) ] = H_4 + H_5 + H_6 + H_7 + H_8, \]
where 
    \begin{align*}
		H_4 ={}& \E \left [ \pi(X) \frac{A R (Y - \tilde m_1(X, S))}{e(X) r(1, X, S)}  \cdot s(A, X,  S, Y, R; \theta_0) \right ] \\
			={}&  \E\Biggl [\pi(X) \frac{A R (Y - \tilde m_1(X, S))}{e(X) r(1, X, S)} \cdot s_1(Y|X, S)  \Biggr ] \\
			 ={}& \E \left [  \pi(X)  \E \left \{  (Y(1) - \tilde m_1(X, S)) \cdot  s_1(Y(1) |X, S)     \Big | X, S \right \}   \right ]   \\
			={}&  \E \left [  \pi(X)  \E \left \{  Y(1) \cdot  s_1(Y(1) |X, S)     \Big | X, S \right \}   \right ]  = \text{ the first term of }  \frac{\partial \V(\pi, \theta; y)}{\partial \theta} \Big|_{\theta = \theta_0}, 
	\end{align*}	
    \begin{align*}
		H_5 ={}& \E \left [  (1 - \pi(X)) \frac{(1-A) R (Y - \tilde m_0(X, S))}{(1-e(X)) r(0, X, S)}    \cdot s(A, X,  S, Y, R; \theta_0) \right ] \\
			={}&  \E\Biggl [  (1 - \pi(X)) \frac{(1-A) R (Y - \tilde m_0(X, S))}{(1-e(X)) r(0, X, S)}   \cdot s_0(Y|X, S)  \Biggr ] \\
			={}&  \E \left [  \pi(X)  \E \left \{  Y(0) \cdot  s_1(Y(0) |X, S)     \Big | X, S \right \}   \right ]  = \text{ the second term of }  \frac{\partial \V(\pi, \theta; y)}{\partial \theta} \Big|_{\theta = \theta_0}, 
	\end{align*}	
    \begin{align*}
		H_6 ={}& \E \left [ \pi(X) \frac{A  (  \tilde m_1(X, S) - m_1(X))}{e(X)}   \cdot s(A, X,  S, Y, R; \theta_0) \right ] \\
		={}&	 \E \left [ \pi(X) \frac{A  (  \tilde m_1(X, S) - m_1(X))}{e(X)}   \cdot   A s_1(S |X)  \right ] \\
			 ={}& \E \left [  \pi(X) \cdot  \E \left \{  (  \tilde m_1(X, S) - m_1(X)) \cdot  s_1(S(1) |X)     \Big | X \right \}   \right ]   \\
			={}&  \E \left [  \pi(X)  \cdot \E \left \{   \tilde m_1(X, S)  \cdot  s_1(S(1) |X)     \Big | X \right \}   \right ]  = \text{ the third term of }  \frac{\partial \V(\pi, \theta; y)}{\partial \theta} \Big|_{\theta = \theta_0}, 
	\end{align*}	
    \begin{align*}
		H_7 ={}& \E \left [   (1 - \pi(X)) \frac{(1-A) (  \tilde m_0(X, S) - m_0(X)) }{1 - e(X)}    \cdot s(A, X,  S, Y, R; \theta_0) \right ] \\
			={}&  \E\Biggl [   (1 - \pi(X)) \frac{(1-A) (  \tilde m_0(X, S) - m_0(X)) }{1 - e(X)}  \cdot (1-A) s_0(S|X)  \Biggr ] \\
			={}&  \E\Big [ (1-\pi(X)) \cdot  \E\Big \{ \tilde m_0(X, S) \cdot s_0(S(0)|X) \Big | X\Big \} \Big ]   = \text{ the fourth term of }  \frac{\partial \V(\pi, \theta; y)}{\partial \theta} \Big|_{\theta = \theta_0}, 
	\end{align*}	
and    	
    \begin{align*}		
		H_8 ={}& \E \left [   \{ \pi(X) m_1(X) + (1 - \pi(X)) m_0(X) - \V(\pi; y) \}  
 \cdot s(A, X,  S, Y, R; \theta_0) \right ] \\
 ={}& \E \left [  \{ \pi(X) m_1(X) + (1 - \pi(X)) m_0(X) - \V(\pi; y) \} 
 \cdot s(X) \right ] \\ 
	={}&   \E\Big [ s(X) \Big  \{  \pi(X) m_1(X) + (1 - \pi(X)) m_0(X) \Big \}  \Big  ] = \text{ the last term of }  \frac{\partial \V(\pi, \theta; y)}{\partial \theta} \Big|_{\theta = \theta_0}. 
	\end{align*}	
Thus, equation (\ref{A.2}) holds. 
In addition, it can be shown that $\tilde \phi_{y} \in \mathcal{T}$, and thus $\tilde \phi_{y}$ is the efficient influence function of $\V(\pi; y)$.

\end{proof}

\section{Proofs of Propositions  \ref{prop-efficiency}-\ref{prop6-2}, Theorem \ref{Asy}, and Proposition \ref{prop6-4}}  \label{app-b}

{\bf Proposition \ref{prop-efficiency}}. 
 \emph{Under the conditions in Theorem \ref{thm-EIF}, if $S$ is associated with $Y$ given $X$, then the semiparametric efficiency bound of $\V(\pi; y)$ is lower compared to the case where $S \indep Y | X$, and the magnitude of this difference is
 \[  \E\left[ \pi^2(X) \frac{ (1 - r(1, X, S)) \cdot  (\tilde m_1(X, S) -  m_1(X))^2}{e(X) r(1, X, S)}   +  (1 - \pi(X))^2 \frac{ (1 - r(0, X, S)) \cdot (\tilde m_0(X, S) - m_0(X))^2}{(1-e(X)) r(0, X, S)}   \right ].\]
 }
 
\begin{proof}[Proof of  Proposition \ref{prop-efficiency}] 
If $S$ is associated with $Y$ given $X$, then the efficient influence function for $\V(\pi; y)$ is 
    	\begin{align*}
	\tilde \phi_{y} ={}& \pi(X) \frac{A R (Y - \tilde m_1(X, S))}{e(X) r(1, X, S)}  + (1 - \pi(X)) \frac{(1-A) R (Y - \tilde m_0(X, S))}{(1-e(X)) r(0, X, S)}    \\
	    {}& + \pi(X) \frac{A  (  \tilde m_1(X, S) - m_1(X))}{e(X)}  + (1 - \pi(X)) \frac{(1-A) (  \tilde m_0(X, S) - m_0(X)) }{1 - e(X)}  \\
	    {}&+ \{ \pi(X) m_1(X) + (1 - \pi(X)) m_0(X) - \V(\pi; y) \}, 
	\end{align*}
and the semiparametric efficiency bound is 
  \begin{align*}
      \text{Var}(\tilde \phi_{y}) ={}& \E[\{ \pi(X) m_1(X) + (1 - \pi(X)) m_0(X) - \V(\pi; y) \}^2] + \E\left[ \pi^2(X) \frac{A R (Y - \tilde m_1(X, S))^2}{e^2(X) r^2(1, X, S)}  \right ]  \\
      {}&  + \E\left [  (1 - \pi(X))^2 \frac{(1-A) R (Y - \tilde m_0(X, S))^2}{(1-e(X))^2 r^2(0, X, S)}   \right ] + \E\left [ \pi^2(X) \frac{A  (  \tilde m_1(X, S) - m_1(X))^2}{e^2(X)} \right ] \\
      {}& + \E\left [ (1 - \pi(X))^2 \frac{(1-A) (  \tilde m_0(X, S) - m_0(X))^2 }{ (1 - e(X))^2 } \right ]  \\
      ={}& \E[\{ \pi(X) m_1(X) + (1 - \pi(X)) m_0(X) - \V(\pi; y) \}^2] + \E\left[ \pi^2(X) \frac{ \E\{(Y - \tilde m_1(X, S))^2 | X, S, A=1, R=1\} }{e(X) r(1, X, S)}  \right ]  \\
      {}&  + \E\left [  (1 - \pi(X))^2 \frac{ \E\{ (Y- \tilde m_0(X, S))^2 | X, S, A=0, R=1 \} }{(1-e(X)) r(0, X, S)}   \right ] + \E\left [ \pi^2(X) \frac{  \E\{(  \tilde m_1(X, S) - m_1(X))^2 |X\} }{e(X)} \right ] \\
      {}& + \E\left [ (1 - \pi(X))^2 \frac{ \E\{(  \tilde m_0(X, S) - m_0(X))^2 |X\} }{ (1 - e(X)) } \right ]  \\ 
    ={}&  \E[\{ \pi(X) m_1(X) + (1 - \pi(X)) m_0(X) - \V(\pi; y) \}^2] + \E\left[ \pi^2(X) \frac{ (Y - \tilde m_1(X, S))^2  }{e(X) r(1, X, S)}  \right ]  \\
      {}&  + \E\left [  (1 - \pi(X))^2 \frac{ (Y- \tilde m_0(X, S))^2  }{(1-e(X)) r(0, X, S)}   \right ] + \E\left [ \pi^2(X) \frac{  (  \tilde m_1(X, S) - m_1(X))^2  }{e(X)} \right ] \\
      {}& + \E\left [ (1 - \pi(X))^2 \frac{ (  \tilde m_0(X, S) - m_0(X))^2  }{ (1 - e(X)) } \right ],  \\ 
  \end{align*}
where the first equality holds because the covariance terms are 0, the second equality follows by law of iterated expectations, and the third equality follows by Assumptions \ref{assump5-1}-\ref{assump5-2}. Likewise, if $S \indep Y | X$,  then $\tilde m_a(X, S) = m_a(X)$, the efficient influence function for $\V(\pi; y)$  simplifies to 
    	\begin{align*}
	\bar \phi_{y} ={}& \pi(X) \frac{A R (Y - m_1(X))}{e(X) r(1, X, S)}  + (1 - \pi(X)) \frac{(1-A) R (Y - m_0(X))}{(1-e(X)) r(0, X, S)}    \\
	    {}&+ \{ \pi(X) m_1(X) + (1 - \pi(X)) m_0(X) - \V(\pi; y) \}, 
	\end{align*}
and the semiparametric efficiency bound is 
   \begin{align*}
      \text{Var}(\bar \phi_{y}) ={}& \E[\{ \pi(X) m_1(X) + (1 - \pi(X)) m_0(X) - \V(\pi; y) \}^2] + \E\left[ \pi^2(X) \frac{A R (Y - m_1(X))^2}{e^2(X) r^2(1, X, S)}  \right ]  \\
      {}&  + \E\left [  (1 - \pi(X))^2 \frac{(1-A) R (Y - m_0(X))^2}{(1-e(X))^2 r^2(0, X, S)}   \right ]  \\
 ={}& \E[\{ \pi(X) m_1(X) + (1 - \pi(X)) m_0(X) - \V(\pi; y) \}^2] + \E\left[ \pi^2(X) \frac{ (Y - m_1(X))^2}{e(X) r(1, X, S)}  \right ]  \\
      {}&  + \E\left [  (1 - \pi(X))^2 \frac{  (Y - m_0(X))^2}{(1-e(X)) r(0, X, S)}   \right ]. 
  \end{align*}
Thus, the magnitude of their difference is 
      \begin{align*}
           \text{Var}(\bar \phi_{y})& -  \text{Var}(\tilde \phi_{y}) 
         ={}    \E\left[ \pi^2(X) \frac{ (Y - m_1(X))^2}{e(X) r(1, X, S)}   +  (1 - \pi(X))^2 \frac{  (Y - m_0(X))^2}{(1-e(X)) r(0, X, S)}   \right ]\\   
  {}&- \E\left[ \pi^2(X) \frac{ (Y - \tilde m_1(X, S))^2  }{e(X) r(1, X, S)}  \right ]   - \E\left [  (1 - \pi(X))^2 \frac{ (Y- \tilde m_0(X, S))^2  }{(1-e(X)) r(0, X, S)}   \right ] \\
  {}&- \E\left [ \pi^2(X) \frac{  (  \tilde m_1(X, S) - m_1(X))^2  }{e(X)} \right ] - \E\left [ (1 - \pi(X))^2 \frac{ (  \tilde m_0(X, S) - m_0(X))^2  }{ (1 - e(X)) } \right ]       \\
      ={}&    \E\left[ \pi^2(X) \frac{ (Y - \tilde m_1(X, S) + \tilde m_1(X, S) -  m_1(X))^2}{e(X) r(1, X, S)}   +  (1 - \pi(X))^2 \frac{  (Y - \tilde m_0(X, S) + \tilde m_1(X, S) - m_0(X))^2}{(1-e(X)) r(0, X, S)}   \right ]\\   
  {}&- \E\left[ \pi^2(X) \frac{ (Y - \tilde m_1(X, S))^2  }{e(X) r(1, X, S)}  \right ]   - \E\left [  (1 - \pi(X))^2 \frac{ (Y- \tilde m_0(X, S))^2  }{(1-e(X)) r(0, X, S)}   \right ] \\
  {}&- \E\left [ \pi^2(X) \frac{  (  \tilde m_1(X, S) - m_1(X))^2  }{e(X)} \right ] - \E\left [ (1 - \pi(X))^2 \frac{ (  \tilde m_0(X, S) - m_0(X))^2  }{ (1 - e(X)) } \right ]     \\
     ={}&    \E\left[ \pi^2(X) \frac{ (\tilde m_1(X, S) -  m_1(X))^2}{e(X) r(1, X, S)}   +  (1 - \pi(X))^2 \frac{  (\tilde m_0(X, S) - m_0(X))^2}{(1-e(X)) r(0, X, S)}   \right ]\\   
  {}&- \E\left [ \pi^2(X) \frac{  (  \tilde m_1(X, S) - m_1(X))^2  }{e(X)} \right ] - \E\left [ (1 - \pi(X))^2 \frac{ (  \tilde m_0(X, S) - m_0(X))^2  }{ (1 - e(X)) } \right ]     \\
  ={}&   \E\left[ \pi^2(X) \frac{ (1 - r(1, X, S)) \cdot  (\tilde m_1(X, S) -  m_1(X))^2}{e(X) r(1, X, S)}   +  (1 - \pi(X))^2 \frac{ (1 - r(0, X, S)) \cdot (\tilde m_0(X, S) - m_0(X))^2}{(1-e(X)) r(0, X, S)}   \right ],
      \end{align*} 
 which leads to the conclusion by noting that $\pi(X)^2 =\pi(X)$ and $(1-\pi(X))^2 = 1-\pi(X)$.
\end{proof}

{\bf Proposition \ref{prop6-2}} (Unbiasedness). \emph{We have that} 

\emph{(a) (Double Robustness).  $\hat  \V(\pi; s)$ is an unbiased estimator of $ \V(\pi; s)$ if one of the following conditions is satisfied: 
 \begin{itemize}
     \item[(i)] $\hat e(x) =e(x)$, i.e., $\hat e(x)$ estimates $e(x)$ accurately;  
     \item[(ii)] $\hat \mu_a(x) = \mu(x)$ i.e., $\mu_a(x)$ estimates $\mu_a(x)$ accurately. 
 \end{itemize}}

\emph{(b) (Quadruple Robustness). $\hat  \V(\pi; y)$ is  an unbiased   estimator of $ \V(\pi; y)$ if one of the following conditions is satisfied: 
	\begin{itemize}
		\item[(i)] $\hat e(x) = e(x)$ and $\hat {\tilde m}_a(x, s) =\tilde m_a(x, s)$; 
		\item[(ii)] $\hat e(x) = e(x)$ and $\hat r(a, x, s) = r(a, x, s)$;    
  \item[(iii)] $\hat m_a(x) = m_a(x)$ and $\hat {\tilde m}_a(x, s) = \tilde m_a(x, s)$; 
		\item[(iv)] $\hat m_a(x) = m_a(x)$ and $\hat r(a, x, s) = r(a, x, s)$. 
	\end{itemize}}

\begin{proof}[Proof of Proposition \ref{prop6-2}] 
    Recall that $Z = (X, A, S, Y)$,  
    \begin{align*}
	        \phi_s(Z; e, \mu_0, \mu_1) &={}   \{ \pi(X) \mu_1(X) + (1 - \pi(X)) \mu_0(X) \} +  \frac{\pi(X) A(S - \mu_1(X))}{e(X)}  +  \frac{(1 - \pi(X)) (1-A)(S - \mu_0(X))}{1 - e(X)},
	\end{align*}
	\begin{align*}
	\phi_{y}(Z;  e, r,  & m_0, m_1, \tilde m_0, \tilde m_1) ={} \{ \pi(X) m_1(X) + (1 - \pi(X)) m_0(X) \} \\
 +{}&  \frac{\pi(X) A R (Y - \tilde m_1(X, S))}{e(X) r(1, X, S)}  + \frac{\pi(X) A  ( \tilde  m_1(X, S) - m_1(X))}{e(X)}   \\
	  +{}&   \frac{(1 - \pi(X))(1-A) R (Y - \tilde  m_0(X, S))}{(1-e(X)) r(0, X, S)} 
	    +  \frac{(1 - \pi(X))(1-A) ( \tilde m_0(X, S) - m_0(X)) }{1 - e(X)},
	\end{align*}
and 
\begin{align*}
\hat   \V(\pi; s)  ={}&  \frac 1 n  \sum_{i=1}^n \phi_s(Z_i; \hat e, \hat \mu_0, \hat \mu_1),  \\
\hat  \V(\pi; y) ={}&  \frac 1 n  \sum_{i=1}^n \phi_{y}(Z_i; \hat e, \hat r, \hat m_0, \hat m_1, \hat{\tilde m}_0, \hat{\tilde m}_1). 
\end{align*}
Next, we prove (a) and (b) separately.

\medskip 
{\bf Proof of (a).} Due to the sample splitting, $\hat e(X_i)$ and $\mu_a(X_i)$ can be seen as an function of $X_i$ when taking expectation of $\E[\hat \V(\pi; s)]$. Thus, 
      \begin{align*}
        \E[&\hat \V(\pi; s)] ={} \E[ \phi_s(Z; \hat e, \hat \mu_0, \hat \mu_1) ]        \\
                      ={}&     \E \left [     \{ \pi(X) \hat \mu_1(X) + (1 - \pi(X)) \hat \mu_0(X) \} +  \frac{\pi(X) A(S - \hat \mu_1(X))}{\hat e(X)}  +  \frac{(1 - \pi(X)) (1-A)(S - \hat \mu_0(X))}{1 -\hat e(X)} \right ]  \\
                      ={}&        \E \left [     \{ \pi(X) \hat \mu_1(X) + (1 - \pi(X)) \hat \mu_0(X) \} +  \frac{\pi(X) A(S(1) - \hat \mu_1(X))}{\hat e(X)}  +  \frac{(1 - \pi(X)) (1-A)(S(0) - \hat \mu_0(X))}{1 -\hat e(X)} \right ]   \\
                      ={}&     \E \left [     \{ \pi(X) \hat \mu_1(X) + (1 - \pi(X)) \hat \mu_0(X) \} \right ] \\
                      & +    \E \left [ \frac{\pi(X) e(X) (\mu_1(X) - \hat \mu_1(X))}{\hat e(X)}  +  \frac{(1 - \pi(X)) (1-e(X))(\mu_0(X) - \hat \mu_0(X))}{1 -\hat e(X)} \right ] , 
      \end{align*}
 where the last equality follows by the law of iterated expectations and Assumption \ref{assump5-1}. 

If $\hat e(X) = e(X)$, $\E[\hat \V(\pi; s)] $ reduces to 
   \begin{align*}
        &    \E \left [     \{ \pi(X) \hat \mu_1(X) + (1 - \pi(X)) \hat \mu_0(X) \} \right ] +    \E \left [ \pi(X) (\mu_1(X) - \hat \mu_1(X))  +  (1 - \pi(X)) (\mu_0(X) - \hat \mu_0(X)) \right ] \\
                        ={}&      \E \left [ \pi(X) \mu_1(X)   +  (1 - \pi(X)) \mu_0(X)  \right ] \\
                        ={}&     \V(\pi; s).                   \qquad \text{By equation (\ref{eq4})}
    \end{align*}    
This proves the conclusion (a)(i). 

If $\hat \mu_a(X) = \mu_a(X)$ for $a = 0, 1$, $\E[\hat \V(\pi; s)] $ reduces to 
  \begin{align*}
     &     \E \left [     \{ \pi(X) \hat \mu_1(X) + (1 - \pi(X)) \hat \mu_0(X) \} \right ] + 0 \\
                   ={}&       \E \left [ \pi(X) \mu_1(X)   +  (1 - \pi(X)) \mu_0(X)  \right ] \\
                   ={}&   \V(\pi; s). 
      \end{align*}
This proves the conclusion (a)(ii).

\medskip 
{\bf Proof of (b).}  Similar to the proof of (a), we first calculate the expectation of $\hat \V(\pi; y)$. 
 \begin{align*}
        & \E[\hat \V(\pi; y)] \\
        ={}& \E[ \phi_y(Z;   \hat e, \hat r, \hat m_0, \hat m_1, \hat{\tilde m}_0, \hat{\tilde m}_1 ]        \\
                      ={}&     \E \left [   \pi(X) \hat m_1(X) + (1 - \pi(X)) \hat m_0(X) \right ] \\
 +{}& \E \left [  \frac{\pi(X) A R (Y - \hat {\tilde m}_1(X, S))}{\hat e(X) \hat r(1, X, S)}  + \frac{\pi(X) A  ( \hat{\tilde  m}_1(X, S) - \hat m_1(X))}{\hat e(X)} \right ]  \\
	  +{}&    \E \left [  \frac{(1 - \pi(X))(1-A) R (Y - \hat{\tilde  m}_0(X, S))}{(1-\hat e(X)) \hat r(0, X, S)} 
	    +  \frac{(1 - \pi(X))(1-A) ( \hat{\tilde m}_0(X, S) - \hat m_0(X)) }{1 -\hat e(X)} \right ] \\
	      ={}&     \E \left [   \pi(X) \hat m_1(X) + (1 - \pi(X)) \hat m_0(X) \right ] \\
 +{}& \E \left [  \frac{\pi(X) \cdot \E[ A R (Y - \hat {\tilde m}_1(X, S)) \mid X, S ] }{\hat e(X) \hat r(1, X, S)}  + \frac{\pi(X) e(X)  ( \hat{\tilde  m}_1(X, S) - \hat m_1(X))}{\hat e(X)} \right ]  \\
	  +{}&    \E \left [  \frac{(1 - \pi(X)) \cdot \E[ (1-A) R (Y - \hat{\tilde  m}_0(X, S)) \mid X, S]}{(1-\hat e(X)) \hat r(0, X, S)} 
	    +  \frac{(1 - \pi(X))(1-e(X)) ( \hat{\tilde m}_0(X, S) - \hat m_0(X)) }{1 -\hat e(X)} \right ] \\
	        ={}&     \E \left [   \pi(X) \hat m_1(X) + (1 - \pi(X)) \hat m_0(X) \right ] \\
 +{}& \E \left [  \frac{\pi(X) \cdot \E[ (Y - \hat {\tilde m}_1(X, S)) \mid X, S, A=1, R=1 ] \cdot \P(A=1, R=1|X,S)  }{\hat e(X) \hat r(1, X, S)}  + \frac{\pi(X) e(X)  ( \hat{\tilde  m}_1(X, S) - \hat m_1(X))}{\hat e(X)} \right ]  \\
	  +{}&    \E \left [  \frac{(1 - \pi(X)) \cdot \E[ ((Y - \hat{\tilde  m}_0(X, S)) \mid X, S, A=0, R=1] \cdot \P(A=0, R=1|X, S)}{(1-\hat e(X)) \hat r(0, X, S)} \right ] \\ 
	    +{}&   \E \left [  \frac{(1 - \pi(X))(1-e(X)) ( \hat{\tilde m}_0(X, S) - \hat m_0(X)) }{1 -\hat e(X)} \right ] \\
	          ={}&     \E \left [   \pi(X) \hat m_1(X) + (1 - \pi(X)) \hat m_0(X) \right ] \\
 +{}& \E \left [  \frac{\pi(X) \cdot \{\tilde m_1(X, S) - \hat {\tilde m}_1(X, S)  \} \cdot  e(X) r(1, X, S)  }{\hat e(X) \hat r(1, X, S)}  + \frac{\pi(X) e(X)  ( \hat{\tilde  m}_1(X, S) - \hat m_1(X))}{\hat e(X)} \right ]  \\
	  +{}&    \E \left [  \frac{(1 - \pi(X)) \cdot  \{\tilde m_0(X,S) - \hat{\tilde  m}_0(X, S)\} \cdot (1 - e(X)) r(0, X, S) }{(1-\hat e(X)) \hat r(0, X, S)} + \frac{(1 - \pi(X))(1-e(X)) ( \hat{\tilde m}_0(X, S) - \hat m_0(X)) }{1 -\hat e(X)} \right ].
      \end{align*}
 
If   $\hat e(X) = e(X)$ and $\hat {\tilde m}_a(X, S) =\tilde m_a(X, S)$, $\E[\hat \V(\pi; y)] $ reduces to 
    \begin{align*}
              &     \E \left [   \pi(X) \hat m_1(X) + (1 - \pi(X)) \hat m_0(X) \right ] \\
 +{}& \E \left [    \pi(X) ( {\tilde  m}_1(X, S) - \hat m_1(X)) +  (1 - \pi(X)) ( {\tilde m}_0(X, S) - \hat m_0(X))  \right ] \\
    ={}& \E[ \pi(X)  \tilde m_1(X, S) + (1 - \pi(X))  \tilde m_0(X, S) ] \\
    ={}&  \V(\pi; y).      \qquad \text{by Proposition \ref{prop5-3}}
    \end{align*}
 This proves (b)(i). 
 
If  $\hat e(X) = e(X)$ and $\hat r(a, X, S) = r(a, X, S)$,     $\E[\hat \V(\pi; y)] $ reduces to 
  \begin{align*}
  &     \E \left [   \pi(X) \hat m_1(X) + (1 - \pi(X)) \hat m_0(X) \right ] \\
 +{}& \E \left [  \pi(X) \cdot \{\tilde m_1(X, S) - \hat {\tilde m}_1(X, S)  \}   +  \pi(X)   ( \hat{\tilde  m}_1(X, S) - \hat m_1(X)) \right ]  \\
	  +{}&    \E \left [  (1 - \pi(X)) \cdot  \{\tilde m_0(X,S) - \hat{\tilde  m}_0(X, S)\}  + (1 - \pi(X)) ( \hat{\tilde m}_0(X, S) - \hat m_0(X))  \right ] \\
	  ={}&   \E[ \pi(X)  \tilde m_1(X, S) + (1 - \pi(X))  \tilde m_0(X, S) ] \\
    ={}&  \V(\pi; y).   
      \end{align*}
   This proves (b)(ii). 
 
   If  $\hat m_a(X) = m_a(X)$ and $\hat {\tilde m}_a(X, S) = \tilde m_a(X, S)$,     $\E[\hat \V(\pi; y)] $ reduces to 
        \begin{align*}
      &       \E \left [   \pi(X)  m_1(X) + (1 - \pi(X))  m_0(X) \right ] \\
 +{}& \E \left [    \frac{\pi(X) e(X)  ( {\tilde  m}_1(X, S) - m_1(X))}{\hat e(X)} +   \frac{(1 - \pi(X))(1-e(X)) ( {\tilde m}_0(X, S) - m_0(X)) }{1 -\hat e(X)} \right ].
      \end{align*}
Note that $\E[\tilde m_a(X,S) \mid X] = m_a(X)$, leading to that $\E[ \tilde m_a(X, S) - m_a(X) | X ] = 0$ for $a = 0, 1$,   
 $\E[\hat \V(\pi; y)] $ can be further reduced to
        \begin{align*}
                \E \left [   \pi(X)  m_1(X) + (1 - \pi(X))  m_0(X) \right ] =   \E[ \pi(X)  \tilde m_1(X, S) + (1 - \pi(X))  \tilde m_0(X, S) ]  =  \V(\pi; y).  
        \end{align*}     
    This proves (b)(iii).

If  $\hat m_a(X) = m_a(X)$ and $\hat r(a, X, S) = r(a, X, S)$, $\E[\hat \V(\pi; y)] $ reduces to 
         \begin{align*}
             &       \E \left [   \pi(X)  m_1(X) + (1 - \pi(X))  m_0(X) \right ] \\
+{}& \E \left [  \frac{\pi(X) \cdot \{\tilde m_1(X, S) - \hat {\tilde m}_1(X, S)  \} \cdot  e(X)   }{\hat e(X) }  + \frac{\pi(X) e(X)  ( \hat{\tilde  m}_1(X, S) - m_1(X))}{\hat e(X)} \right ]  \\
	  +{}&   \E \left [  \frac{(1 - \pi(X)) \cdot  \{\tilde m_0(X,S) - \hat{\tilde  m}_0(X, S)\} \cdot (1 - e(X))  }{(1-\hat e(X)) } + \frac{(1 - \pi(X))(1-e(X)) ( \hat{\tilde m}_0(X, S) - m_0(X)) }{1 -\hat e(X)} \right ] \\
	        ={}&       \E \left [   \pi(X)  m_1(X) + (1 - \pi(X))  m_0(X) \right ] \\
+{}& \E \left [  \frac{\pi(X) \cdot \{\tilde m_1(X, S) - m_1(X)  \} \cdot  e(X)   }{\hat e(X) }  +  \frac{(1 - \pi(X)) \cdot  \{\tilde m_0(X,S) - m_0(X)\} \cdot (1 - e(X))  }{(1-\hat e(X)) }  \right ] \\
     ={}&   \E \left [   \pi(X)  m_1(X) + (1 - \pi(X))  m_0(X) \right ]  \\
     ={}&  \V(\pi; y). 
      \end{align*} 
    This proves (b)(iv). 
    
\end{proof}

{\bf Theorem \ref{Asy}} (Asymptotic Properties). \emph{We have that} 
   
\emph{(a) if $|| \hat e(x) - e(x)  ||_2 \cdot  || \hat \mu_a(x) - \mu_a(x) ||_2   = o_{\P}(n^{-1/2})$ for all $x\in\mathcal{X}$ and $a\in \{0,  1\}$,  then $\hat \V(\pi; s)$ is a consistent  estimator of $\V(\pi; s)$, and satisfies   
    \[ \sqrt{n} \{\hat \V(\pi; s) -  \V(\pi; s) \}     \xrightarrow{d} N(0,  \sigma_s^2),  \]
  where $\sigma_s^2 = \text{Var}(\phi_s)$ is the semiparametric efficiency bound of $\V(\pi; s)$, and  $\xrightarrow{d}$ means convergence in distribution.}

\emph{(b) if $|| \hat e(x) - e(x)  ||_2 \cdot  || \hat m_a(x) - m_a(x) ||_2   = o_{\P}(n^{-1/2})$ and $|| \hat r(a, x, s) - r(a, x, s)  ||_2 \cdot  || \hat {\tilde m}_a(x, s) - \tilde m_a(x, s) ||_2   = o_{\P}(n^{-1/2})$  for all $x\in\mathcal{X}$, $a\in \{0,  1\}$ and $s \in \mathcal{S}$, then $\hat \V(\pi; y)$ is a consistent estimator of $\V(\pi; y)$, and satisfies 
    \[ \sqrt{n} \{\hat \V(\pi; y) -  \V(\pi; y) \}     \xrightarrow{d} N(0,  \sigma_y^2),  \]
  where $\sigma_y^2$ is the semiparametric efficiency bound of $\V(\pi; y)$.}

\begin{proof}[Proof of Theorem \ref{Asy}]
We prove Theorem \ref{Asy}(a) and Theorem \ref{Asy}(b), separately. 

 {\bf Proof of (a).}  Recall that 
      $Z = (X, A, S, Y)$,  
    \begin{align*}
	        \phi_s(Z; e, \mu_0, \mu_1) &={}   \{ \pi(X) \mu_1(X) + (1 - \pi(X)) \mu_0(X) \} +  \frac{\pi(X) A(S - \mu_1(X))}{e(X)}  +  \frac{(1 - \pi(X)) (1-A)(S - \mu_0(X))}{1 - e(X)},
	\end{align*} 
and $\hat   \V(\pi; s)  =  \frac 1 n  \sum_{i=1}^n \phi_s(Z_i; \hat e, \hat \mu_0, \hat \mu_1)$. In addition, $ \V(\pi; s) = \E[ \phi_s(Z; e, \mu_0, \mu_1)]$. 
   
 The estimator $\hat   \V(\pi; s)$ can be decomposed as  
	\begin{align*}
  \hat   \V(\pi; s) -   \V(\pi; s) =  U_{1n} + U_{2n},  
	\end{align*}
where 
    \begin{align*}
    	  U_{1n} ={}&  \frac 1 n \sum_{i=1}^n [   \phi_s(Z_i; e, \mu_0, \mu_1) -    \V(\pi; s) ], \\
        U_{2n} ={}& \frac 1 n \sum_{i=1}^n [   \phi_s(Z_i; \hat e, \hat \mu_0, \hat \mu_1)  -   \phi_s(Z_i; e, \mu_0, \mu_1) ]. 
    \end{align*}
Note that $U_{1n}$ is a sum of $n$ independent variables with zero means,  and its variance equals $\sigma_s^2$.  
By the central limit theorem, 
        \[ \sqrt{n} U_{1n}     \xrightarrow{d}  N\left (0,  \sigma_s^2 \right ).  \]  
 Thus, it suffices to show that $U_{2n} = o_{\P}(n^{-1/2})$.

Next, we focus on analyzing $U_{2n}$, which can be be further decomposed as 
        \[  U_{2n} =  U_{2n} - \bfE[U_{2n}] +  \bfE[U_{2n}].    \]
 
Define the Gateaux derivative of the generic function $g(Z; e, \mu_0, \mu_1)$ in the direction $[\hat e - e,  \hat \mu_0 - \mu_0, \hat \mu_1 - \mu_1]$ as 
 \begin{align*}
  & \partial_{[\hat e - e,  \hat \mu_0 - \mu_0, \hat \mu_1 - \mu_1]}g(Z; e, \mu_0, \mu_1)   \\
 ={}&  \frac{ \partial g(Z; e + \alpha_1 (\hat e - e), \mu_0, \mu_1) } { \partial \alpha_1 }\Big |_{\alpha_1 = 0} +    \frac{ \partial g(Z; e, \mu_0 + \alpha_2 (\hat \mu_0 -  \mu_0), \mu_1) } { \partial \alpha_2 }\Big |_{\alpha_2 = 0} +  \frac{ \partial g(Z; e, \mu_0, \mu_1+\alpha_3 (\hat \mu_1 -  \mu_1)) } { \partial \alpha_2 }\Big |_{\alpha_2 = 0}.
 \end{align*}   
   By a Taylor expansion for $\bfE[U_{2n}]$ yields that  
   \begin{align*}
 \bfE[&U_{2n}] ={}   \bfE [ \phi_s(Z; \hat e, \hat \mu_0, \hat \mu_1)  -   \phi_s(Z; e, \mu_0, \mu_1) ]  \\
     ={}& \partial_{[\hat e - e,  \hat \mu_0 - \mu_0, \hat \mu_1 - \mu_1]} \bfE[   \phi_s(Z; e, \mu_0, \mu_1)   ] +  \frac 1 2  \partial^2_{[\hat e - e,  \hat \mu_0 - \mu_0, \hat \mu_1 - \mu_1]} \bfE[   \phi_s(Z; e, \mu_0, \mu_1) ] + \cdots
   \end{align*}
The first-order term  
	\begin{align*}
	& \partial_{[\hat e - e,  \hat \mu_0 - \mu_0, \hat \mu_1 - \mu_1]} \bfE[    \phi_s(Z; e, \mu_0, \mu_1)   ]  \\
	={}& \bfE \biggl [   \left \{  \frac{ (1 -\pi(X)) (1-A) \{ S - \mu_0(X)\} }{ (1-e(X))^2 }  - \frac{  \pi(X) A\{ S - \mu_1(X)\} }{ e(X)^2 }  \right \} \{\hat e(X) - e(X) \}  \\
	 {}&+   (1 - \pi(X)) \left \{ 1  -  \frac{1- A}{1-  e(X) } \right \}  \{  \hat \mu_0(X) - \mu_0(X) \}   +  \pi(X)  \left \{ 1 -  \frac{A  }{  e(X) } \right \}  \{  \hat \mu_1(X) - \mu_1(X) \}   \biggr ] \\
	={}& 0,
	\end{align*}
where the last equation follows from $\E[A|X] - e(X) = 0$, $\E[ A(S - \mu_1(X)) |X ] = 0$, and $\E[ (1-A)(S - \mu_0(X)) |X ] = 0$. 
  For the second-order term, we get 
    \begin{align*}
       &  \frac 1 2  \partial^2_{[\hat e - e,  \hat \mu_0 - \mu_0, \hat \mu_1 - \mu_1]} \bfE[    \phi_s(Z; e, \mu_0, \mu_1)   ]\\
	={}& \bfE \biggl [  \left \{  \frac{ (1-\pi(X)) (1-A) \{ S - \mu_0(X)\} }{ (1-e(X))^3 } + \frac{\pi(X) A (S - \mu_1(X)) }{ e(X)^3 }  \right \} \{\hat e(X) - e(X) \}^2     \\
	&-     \frac{(1-\pi(X))(1- A)  }{ (1-  e(X))^2 }  \{\hat e(X) - e(X) \}  \{  \hat \mu_0(X) - \mu_0(X) \}     + \frac{\pi(X)A}{e(X)^2} \{\hat e(X) - e(X) \}  \{  \hat \mu_1(X) - \mu_1(X) \}     \biggr ] \\
	={}&   \bfE \biggl [   \frac{\pi(X)A}{e(X)^2} \{\hat e(X) - e(X) \}  \{  \hat \mu_1(X) - \mu_1(X) \}    -     \frac{(1-\pi(X))(1- A)  }{ (1-  e(X))^2 }  \{\hat e(X) - e(X) \}  \{  \hat \mu_0(X) - \mu_0(X) \}      \biggr ] \\
	={}& O_{\P} \biggl (   ||\hat e(X) - e(X) ||_2  \cdot \|  \hat \mu_1(X) - \mu_1(X) ||_2  +   || \hat e(X) - e(X) ||_2 \cdot ||  \hat \mu_0(X) - \mu_0(X) ||_2 \biggr ) \\
	={}& o_{\P}(n^{-1/2}),
     \end{align*}
All higher-order terms can be shown to be dominated by the second-order term. Therefore, 
    $\bfE[U_{2n}]  = o_{\P}(n^{-1/2}).$  
 In addition,  we get that $U_{2n} - \bfE[U_{2n}] = o_{\P}(n^{-1/2)}$ by calculating $\text{Var}\{\sqrt{n}(U_{2n} - \bfE[U_{2n}])\} = o_{\P}(1)$. This proves the conclusion of Theorem \ref{Asy}(a).

 {\bf Proof of (b).}  Recall that \begin{align*}
	\phi_{y}(Z;  e, r,  & m_0, m_1, \tilde m_0, \tilde m_1) ={} \{ \pi(X) m_1(X) + (1 - \pi(X)) m_0(X) \} \\
 +{}&  \frac{\pi(X) A R (Y - \tilde m_1(X, S))}{e(X) r(1, X, S)}  + \frac{\pi(X) A  ( \tilde  m_1(X, S) - m_1(X))}{e(X)}   \\
	  +{}&   \frac{(1 - \pi(X))(1-A) R (Y - \tilde  m_0(X, S))}{(1-e(X)) r(0, X, S)} 
	    +  \frac{(1 - \pi(X))(1-A) ( \tilde m_0(X, S) - m_0(X)) }{1 - e(X)},
	\end{align*}
and 
\begin{align*}
\hat  \V(\pi; y) ={}&  \frac 1 n  \sum_{i=1}^n \phi_{y}(Z_i; \hat e, \hat r, \hat m_0, \hat m_1, \hat{\tilde m}_0, \hat{\tilde m}_1). 
\end{align*}

Similar to the proof of (a), we decompose $ \hat   \V(\pi; y) -   \V(\pi; y) $ as  
	\begin{align*}
  \hat   \V(\pi; y) -   \V(\pi; y) =  U_{3n} + U_{4n},  
	\end{align*}
where 
    \begin{align*}
    	  U_{1n} ={}&  \frac 1 n \sum_{i=1}^n [  \phi_{y}(Z_i;  e, r,   m_0, m_1, \tilde m_0, \tilde m_1)  -    \V(\pi; y) ], \\
        U_{2n} ={}& \frac 1 n \sum_{i=1}^n [  \phi_{y}(Z_i; \hat e, \hat r, \hat m_0, \hat m_1, \hat{\tilde m}_0, \hat{\tilde m}_1)  -   \phi_{y}(Z_i;  e, r,   m_0, m_1, \tilde m_0, \tilde m_1)  ]. 
    \end{align*}
Note that $U_{3n}$ is a sum of $n$ independent variables with zero means,  and its variance equals $\sigma_y^2$.  
By the central limit theorem, 
        \[ \sqrt{n} U_{3n}     \xrightarrow{d}  N\left (0,  \sigma_s^2 \right ).  \]  
 Thus, it suffices to show that $U_{4n} = o_{\P}(n^{-1/2})$.  $U_{4n}$ can be be further decomposed as 
        \[  U_{4n} =  U_{4n} - \bfE[U_{4n}] +  \bfE[U_{4n}].    \]
   By a Taylor expansion for $\bfE[U_{4n}]$ yields that  
   \begin{align*}
 \bfE[&U_{4n}] ={}   \bfE [ \phi_{y}(Z; \hat e, \hat r, \hat m_0, \hat m_1, \hat{\tilde m}_0, \hat{\tilde m}_1)  -   \phi_{y}(Z;  e, r,   m_0, m_1, \tilde m_0, \tilde m_1) ]  \\
     ={}& \partial_{[\hat e - e, \hat r - r,  \hat m_0 - m_0, \hat m_1 - m_1, \hat{\tilde m}_0 - \tilde m_0, \hat{\tilde m}_1 - \tilde m_1  ]} \bfE[   \phi_{y}(Z;  e, r,   m_0, m_1, \tilde m_0, \tilde m_1)   ]  \\
     {}& +  \frac 1 2  \partial^2_{[\hat e - e, \hat r - r,  \hat m_0 - m_0, \hat m_1 - m_1, \hat{\tilde m}_0 - \tilde m_0, \hat{\tilde m}_1 - \tilde m_1 ]} \bfE[   \phi_{y}(Z;  e, r,   m_0, m_1, \tilde m_0, \tilde m_1) ]  \\ 
     {}&+ \cdots
   \end{align*}
The first-order term  
	\begin{align*}
	&  \partial_{[\hat e - e, \hat r - r,  \hat m_0 - m_0, \hat m_1 - m_1, \hat{\tilde m}_0 - \tilde m_0, \hat{\tilde m}_1 - \tilde m_1  ]} \bfE[   \phi_{y}(Z;  e, r,   m_0, m_1, \tilde m_0, \tilde m_1)   ] \\
	={}& \bfE \biggl [   \Big \{  \frac{ (1 -\pi(X)) (1-A) R \{ Y - \tilde m_0(X, S)\} }{ (1-e(X))^2  r(0, X, S)} + \frac{ (1-\pi(X)) (1-A) (\tilde m_0(X, S) - m_0(X)) }{ (1 -e(X))^2 }  \\
	 {}& - \frac{  \pi(X) A R \{ Y - \tilde m_1(X, S)\} }{ e(X)^2 r(1, X, S) } - \frac{  \pi(X) A  \{ \tilde m_1(X, S) - m_1(X)\} }{ e(X)^2 r(1, X, S) }  \Big \} \{\hat e(X) - e(X) \}  \\ 
	 {}&  - \frac{ \pi(X) A R (Y - \tilde m_1(X, S)) }{ e(X) r(1, X, S)^2 } \cdot \{\hat r(1, X, S) - r(1, X, S) \}  \\
	  {}&  - \frac{ (1-\pi(X)) (1- A) R (Y - \tilde m_0(X, S)) }{ (1-e(X)) r(0, X, S)^2 } \cdot \{\hat r(0, X, S) - r(0, X, S) \}  \\
	 {}&+   (1 - \pi(X)) \left \{ 1  -  \frac{1- A}{1-  e(X) } \right \}  \{  \hat m_0(X) - m_0(X) \}  \\
	  {}&+  \pi(X)  \left \{ 1 -  \frac{A  }{  e(X) } \right \}  \{  \hat m_1(X) - m_1(X) \}  \\
	  {}&+ \frac{ (1 -\pi(X))(1-A) }{ 1 - e(X) }  \left \{ 1 - \frac{R}{ r(0, X, S) }  \right \} \{ \hat {\tilde m}_0(X, S) - \tilde m_0(X, S)   \}  \\
	  {}&+  \frac{ \pi(X) A }{ e(X) }  \left \{ 1 - \frac{R}{ r(1, X, S) }  \right \} \{ \hat {\tilde m}_1(X, S) - \tilde m_1(X, S)   \} \biggr ] \\
	={}& 0. 
	\end{align*}
  For the second-order term, we get 
    \begin{align*}
       &  \frac 1 2  \partial^2_{[\hat e - e,  \hat \mu_0 - \mu_0, \hat \mu_1 - \mu_1]} \bfE[    \phi_s(Z; e, \mu_0, \mu_1)   ]\\
	={}& \frac 1 2 \bfE \biggl [   \frac{\pi(X) A }{ e(X)^2 } \cdot  \{\hat e(X) - e(X) \} \cdot \{\hat m_1(X) -   m_1(X) \}   -  \frac{ (1-\pi(X)) (1-A) }{ (1-e(X))^2 } \cdot  \{\hat e(X) - e(X) \} \cdot \{\hat m_0(X) -   m_0(X) \}    \\
	&   +  \frac{\pi(X)A R}{e(X) r(1, X, S)^2 } \{ \hat r(1, X, S) - r(1, X, S) \}  \{  \hat {\tilde m}_1(X,S) - {\tilde m}_1(X,S)  \} \\    &  +   \frac{(1-\pi(X))(1-A) R}{(1-e(X)) r(0, X, S)^2 } \{ \hat r(0, X, S) - r(0, X, S) \}  \{  \hat {\tilde m}_0(X,S) - {\tilde m}_0(X,S)  \}     \\
	&+  \frac{\pi(X) A }{ e(X)^2 } \cdot  \{\hat e(X) - e(X) \} \cdot \{\hat m_1(X) -   m_1(X) \}   -  \frac{ (1-\pi(X)) (1-A) }{ (1-e(X))^2 } \cdot  \{\hat e(X) - e(X) \} \cdot \{\hat m_0(X) -   m_0(X) \}   \\
	&+  \frac{\pi(X)A R}{e(X) r(1, X, S)^2 } \{ \hat r(1, X, S) - r(1, X, S) \}  \{  \hat {\tilde m}_1(X,S) - {\tilde m}_1(X,S)  \} \\    &  +   \frac{(1-\pi(X))(1-A) R}{(1-e(X)) r(0, X, S)^2 } \{ \hat r(0, X, S) - r(0, X, S) \}  \{  \hat {\tilde m}_0(X,S) - {\tilde m}_0(X,S)  \}       \biggr ] \\
	={}& \bfE \biggl [   \frac{\pi(X) A }{ e(X)^2 } \cdot  \{\hat e(X) - e(X) \} \cdot \{\hat m_1(X) -   m_1(X) \}   -  \frac{ (1-\pi(X)) (1-A) }{ (1-e(X))^2 } \cdot  \{\hat e(X) - e(X) \} \cdot \{\hat m_0(X) -   m_0(X) \}    \\
	&   +  \frac{\pi(X)A R}{e(X) r(1, X, S)^2 } \{ \hat r(1, X, S) - r(1, X, S) \}  \{  \hat {\tilde m}_1(X,S) - {\tilde m}_1(X,S)  \} \\    &  +   \frac{(1-\pi(X))(1-A) R}{(1-e(X)) r(0, X, S)^2 } \{ \hat r(0, X, S) - r(0, X, S) \}  \{  \hat {\tilde m}_0(X,S) - {\tilde m}_0(X,S)  \}      \biggr ] \\  
	={}& O_{\P} \biggl (   ||\hat e(X) - e(X) ||_2  \cdot \|  \hat m_1(X) - m_1(X) ||_2  +   || \hat e(X) - e(X) ||_2 \cdot ||  \hat m_0(X) - m_0(X) ||_2  \\
	{}&  +   || \hat r(0, X, S) - r(0, X, S) ||_2 \cdot  ||  \hat {\tilde m}_0(X,S) - {\tilde m}_0(X,S)  ||_2 + || \hat r(1, X, S) - r(1, X, S) ||_2 \cdot  ||  \hat {\tilde m}_1(X,S) - {\tilde m}_1(X,S)  ||_2    \biggr ) \\
	={}& o_{\P}(n^{-1/2}),
     \end{align*}
All higher-order terms can be shown to be dominated by the second-order term. Therefore, 
    $\bfE[U_{4n}]  = o_{\P}(n^{-1/2}).$  
 In addition,  we get that $U_{4n} - \bfE[U_{4n}] = o_{\P}(n^{-1/2)}$ by calculating $\text{Var}\{\sqrt{n}(U_{4n} - \bfE[U_{4n}])\} = o_{\P}(1)$. This proves the conclusion of Theorem \ref{Asy}(b).
 
\end{proof}

Next, we give the detailed proof of Proposition \ref{prop6-4}, which relies on the following Lemma \ref{shapiro}. 

\begin{lemma}\cite{shapiro1991asymptotic}\label{shapiro}
Let $\Theta$ be a compact subset of $\mathbb{R}^k$. Let $C(\Theta)$ denote the set of continuous real-valued functions on $\Theta$, with $\mathcal{L}=C(\Theta) \times \ldots \times C(\Theta)$ the $r$-dimensional Cartesian product. Let $\psi(\theta)=\left(\psi_0, \ldots, \psi_r\right) \in \mathcal{L}$ be a vector of convex functions. Consider the quantity $\alpha^*$ defined as the solution to the following convex optimization program:
$$
\begin{aligned}
& \alpha^*=\min _{\theta \in \Theta} \psi_0(\theta) \\
& \text { subject to } \psi_j(\theta) \leq 0, j=1, \ldots, r
\end{aligned}
$$
Assume that Slater's condition holds, so that there is some $\theta \in \Theta$ for which the inequalities are satisfied and non-affine inequalities are strictly satisfied, i.e. $\psi_j(\theta)<0$ if $\psi_j$ is non-affine. Now consider a sequence of approximating programs, for $n=1,2, \ldots$ :
$$
\begin{aligned}
& \widehat{\alpha}_n=\min _{\theta \in \Theta} \quad \widehat{\psi}_{0 n}(\theta) \\
& \text { subject to } \widehat{\psi}_{j n}(\theta) \leq 0, j=1, \ldots, r
\end{aligned}
$$
with $\widehat{\psi}_n(\theta):=\left(\widehat{\psi}_{0 n}, \ldots, \widehat{\psi}_{r n}\right) \in \mathcal{L}$. Assume that $f(n)\left(\widehat{\psi}_n-\psi\right)$ converges in distribution to a random element $W \in \mathcal{L}$ for some real-valued function $f(n)$. Then:
$$
f(n)\left(\widehat{\alpha}_n-\alpha^*\right) \rightsquigarrow L
$$
for a particular random variable L. It follows that $\widehat{\alpha}_n-\alpha^*=O_{\mathbb{P}}(1 / f(n))$.
\end{lemma}
\hfill $\Box$

{\bf Proposition \ref{prop6-4}} (Regret and Estimation Error). 
\emph{Suppose that for all $\pi\in\Pi$, $\pi(x)=\pi(x;\theta)$ is a continuously differentiable and convex function with respect to $\theta$, under the conditions in Theorem \ref{Asy},}  we have
 
 \emph{(a) The expected reward of the learned policy is consistent, and $U(\hat \pi^*) - U(\pi^*) = O_{\mathbb{P}}(1 /\sqrt{n})$;}

\emph{(b) The estimated reward of the learned policy is consistent, and 
$\hat U(\hat \pi^*) - U(\pi^*)= O_{\P}(1/\sqrt{n})$.}  

\begin{proof}[Proof of Proposition \ref{prop6-4}] 
We first show Proposition \ref{prop6-4}(b). According to the proof of Theorem \ref{Asy}, we have 
\begin{align*}
   \sqrt{n}\{\hat U(\pi) - U(\pi)\} ={}& \sqrt{n} \{\hat \V(\pi; s) -   \V(\pi; s) \} + \lambda \sqrt{n} \{ \hat \V(\pi; y)- \V(\pi; y)\} \\
   ={}&  \frac{1}{\sqrt{n}} \sum_{i=1}^n [   \phi_s(Z_i; e, \mu_0, \mu_1) -    \V(\pi; s) ]  + \lambda \frac{1}{\sqrt{n}} \sum_{i=1}^n [  \phi_{y}(Z_i;  e, r,   m_0, m_1, \tilde m_0, \tilde m_1)  -    \V(\pi; y) ] + o_{\P}(1). 
\end{align*}
By the central limit theorem, 
    \begin{equation*} 
         \sqrt{n}\{\hat U(\pi) - U(\pi)\}  \xrightarrow{d} N\left ( 0,  \sigma^2  \right ), 
    \end{equation*}  
where 
    $$\sigma^2 = \text{Var}\Big(  \phi_s(Z_i; e, \mu_0, \mu_1) + \lambda \phi_{y}(Z_i;  e, r,   m_0, m_1, \tilde m_0, \tilde m_1)  -     \V(\pi; s) - \lambda \V(\pi; y) \Big ).$$
This implies that for any given $\pi$, 
\begin{equation} \label{eq-A3}
\hat U(\pi) - U(\pi) = O_{\P}(n^{-1/2}).
\end{equation}

Under Assumptions that for all $\pi\in\Pi$, $\pi(x)=\pi(x;\theta)$ is a continuously differentiable and convex function with respect to $\theta$. The convexity of $\hat   U(\pi)$  follows directly from the convexity of $\pi(x)=\pi(x;\theta)$ with respect to $\theta$, and the linearity of $\hat U(\pi)$ with respect to $\pi\in\Pi$. 
Then the conclusion of Proposition \ref{prop6-4}(b) follows from the direct application of Lemma \ref{shapiro}, and $f(n)=\sqrt{n}$.

Next, we prove Proposition \ref{prop6-4}(a). Note that 
\[ U(\hat \pi^*) - U(\pi^*) = \{ U(\hat \pi^*) -  \hat U(\hat \pi^*)\} + \{ \hat U(\hat \pi^*)  - U(\pi^*)\},     \]
the first term of the right side is $O_{\P}(1/\sqrt{n})$ by equation (\ref{eq-A3}),  and the second term of the right side also is $O_{\P}(1/\sqrt{n})$ by Proposition \ref{prop6-4}(b). Thus, $U(\hat \pi^*) - U(\pi^*)  = O_{\P}(1/\sqrt{n})$. This finishes the proof. 

\end{proof}

\section{Estimation of Nuisance Parameters with Sample Splitting} \label{app-c}

Let $K$ be a small positive integer, and (for simplicity) suppose that $m = n / K$ is also an integer. Let
  $I_{1}$, ..., $I_{K}$ be a random partition of the index set $I = \{1, ...,  n\}$ so that $\# I_{k} = m$ for $k = 1, ..., K$. Denote $I_{k}^{C}$ as the complement of $I_{k}$.

  \medskip  \noindent
~ {\bf Step 1.} 	Nuisance parameter training for each sub-sample.

 \medskip
 \quad {\bf for } $k = 1$ {\bf to} $K$ {\bf do}
 		
\qquad	~	(1) Construct estimators of $ e(X)$, $ r(a, X, S)$, $ \mu_a(X)$, $ m_a(X)$, and $ \tilde m_a(X,S)$ for $a = 0, 1$,
using the sample with $I_{k}^{C}$. The associated estimators are denoted as $\check e(x)$, $\check r(a, X, S)$, $\check \mu_a(X)$, $\check m_a(X)$, and $\check {\tilde m}_a(X,S)$ for $a = 0, 1$. 

\qquad	~ (2) Obtain the predicted values of $\check e(X_i)$, $\check r(a, X_i, S_i)$, $\check \mu_a(X_i)$, $\check m_a(X_i)$, and $\check {\tilde m}_a(X_i,S_i)$ for $i \in I_k$. 

 \medskip
 \quad {\bf end}

 \medskip  \noindent
~ {\bf Step 2.} All the predicted values  $\check e(X_i)$, $\check r(a, X_i, S_i)$, $\check \mu_a(X_i)$, $\check m_a(X_i)$, and $\check {\tilde m}_a(X_i,S_i)$  for $i \in I$ consist of the final estimates of $ e(X)$, $ r(a, X, S)$, $ \mu_a(X)$, $ m_a(X)$, and $ \tilde m_a(X,S)$, denoted as  $\hat e(X_i)$, $\hat r(a, X_i, S_i)$, $\hat \mu_a(X_i)$, $\hat m_a(X_i)$, and $\hat {\tilde m}_a(X_i,S_i)$,  respectively. 

 \medskip  \noindent
~ {\bf Step 3.} The estimators of short-term and long-term rewards are given as 
\begin{align*} 
\begin{split}
\hat   \V(\pi; s)  ={}&  \frac 1 n  \sum_{i=1}^n \phi_s(Z_i; \hat e, \hat \mu_0, \hat \mu_1),  \\
\hat  \V(\pi; y) ={}&  \frac 1 n  \sum_{i=1}^n \phi_{y}(Z_i; \hat e, \hat r, \hat m_0, \hat m_1, \hat{\tilde m}_0, \hat{\tilde m}_1).     
\end{split}
\end{align*}

\medskip 
The full sample is split into $K$ parts, the short-term and long-term rewards are estimated for each subsample, while the 
nuisance parameter training is implemented in the corresponding complement sample. The resulting estimators of short-term and long-term rewards are the average values of the estimators in each subsample.  This is the ``cross-fitting'' approach widely used in causal inference \cite{Chernozhukov-etal-2018, wager2018estimation, Athey-Tibsirani-Wager-2019, Semenova-Chernozhukov}.

Note that when estimating $\mathbb{V}(\pi; y)$, $\mathbb{V}(\pi; s)$ using different estimators (Proposed, IPW, OR, DM), we use the widely-used Adam optimization method to learn the policy for this unconstrained problem.

\section{Additional Experimental Results} \label{app:more_exp}
In the following, we show more experimental results with different missing ratios $\{ 0.2, 0.3, 0.4, 0.5, 0.6\}$ under \textsc{IHDP} and \textsc{JOBS} datasets, in Tables~\ref{tab:fixed_missing02}-\ref{tab:fixed_missing06}. Further, we show more experimental results with different time steps when the missing ratio varies under \textsc{IHDP} and \textsc{JOBS} datasets, in Figure~\ref{fig:difftime0104}.

\begin{table*}[t!]
\centering
\caption{Comparison of the baselines, \textsc{Naive-S} (maximizing short-term rewards alone), \textsc{Naive-Y} (maximizing long-term rewards alone), and the proposed method in terms of the rewards, welfare changes, and policy errors on \textsc{IHDP} and \textsc{Jobs}, based on our proposed estimator. Different balance factors are employed for the estimation and evaluation, $\lambda=0, 0.5, 1$, where the expected short-term and long-term rewards are estimated by outcome regression and multi-layer perceptron regression methods. Higher reward/$\Delta$W and lower error mean better performance. The missing ratio is 0.1. }
\begin{sc}
\begin{tabular}{l | lll | lll | lll}
  \toprule
   \multicolumn{1}{c|}{ \textsc{IHDP}}  
   & \multicolumn{3}{c|}{Short-term metrics}  
   & \multicolumn{3}{c|}{\textbf{Balanced metrics}}
   & \multicolumn{3}{c}{Long-term metrics}\\
  \midrule
  \multicolumn{1}{l|}{Methods}   & 
  \multicolumn{1}{c}{Reward} & 
  \multicolumn{1}{c}{$\Delta$W} & \multicolumn{1}{c|}{ error} & 
  \multicolumn{1}{c}{Reward} & 
  \multicolumn{1}{c}{$\Delta$W} & \multicolumn{1}{c|}{error} &
  \multicolumn{1}{c}{Reward} & 
  \multicolumn{1}{c}{$\Delta$W} & \multicolumn{1}{c}{error} \\  
  \midrule
  \multirow{1}{*}{Naive-S}
  & \textbf{534.7} &\textbf{87.8}& 0.494 & 1315.9 & 473.2 & 0.498 & 781.2 & 770.9 & 0.500 \\
  \multirow{1}{*}{Naive-Y} 
  & 529.7 & 82.8 & \textbf{0.482} & 2225.1 & 925.3 & 0.398 & 1695.4 & 1685.0 & 0.399 \\
  \multirow{1}{*}{Ours}
  & 529.3 & 82.4 & 0.486 & 
  \textbf{2272.4} &\textbf{948.7} &\textbf{0.395}  & \textbf{1743.1} &\textbf{1732.8}& \textbf{0.396} \\
  \midrule
  \midrule
\multicolumn{1}{c|}{ \textsc{JOBS}}  
   & \multicolumn{3}{c|}{Short-term metrics}  
   & \multicolumn{3}{c|}{\textbf{Balanced metrics}}
   & \multicolumn{3}{c}{Long-term metrics}\\
  \midrule
  \multicolumn{1}{l|}{Methods}   & 
  \multicolumn{1}{c}{Reward} & 
  \multicolumn{1}{c}{$\Delta$W} & \multicolumn{1}{c|}{ error} & 
  \multicolumn{1}{c}{Reward} & 
  \multicolumn{1}{c}{$\Delta$W} & \multicolumn{1}{c|}{error} &
  \multicolumn{1}{c}{Reward} & 
  \multicolumn{1}{c}{$\Delta$W} & \multicolumn{1}{c}{error} \\  
  \midrule
  \multirow{1}{*}{Naive-S}
  & \textbf{1694.3} & \textbf{419.5} & \textbf{0.469} & 2835.2 & 406.0 & 0.486 & 1140.9 & -27.1 & 0.506\\
  \multirow{1}{*}{Naive-Y}  
  & 1599.3 & 324.6 & 0.510 & 2863.6  & 372.7 & 0.482 & \textbf{1264.2} &\textbf{96.2} &\textbf{0.477}\\
  \multirow{1}{*}{Ours}
  & 1670.4 & 395.6 & 0.479 & \textbf{2912.9} & \textbf{432.9}& \textbf{0.470} & 1242.5 &  74.6 & 0.481 \\
  \bottomrule
  \end{tabular}
  \end{sc}
\label{tab:fixed_missing}%
\end{table*}%

\begin{table*}[t!]
\centering
\caption{Comparison of the baselines, \textsc{Naive-S} (maximizing short-term rewards alone), \textsc{Naive-Y} (maximizing long-term rewards alone), and the proposed method in terms of the rewards, welfare changes, and policy errors on \textsc{IHDP} and \textsc{Jobs}. Higher reward/$\Delta$W and lower error mean better performance. The missing ratio is 0.2.}
\begin{sc}
\begin{tabular}{l | lll | lll | lll}
  \toprule
   \multicolumn{1}{c|}{ \textsc{IHDP}}  
   & \multicolumn{3}{c|}{Short-term metrics}  
   & \multicolumn{3}{c|}{\textbf{Balanced metrics}}
   & \multicolumn{3}{c}{Long-term metrics}\\
  \midrule
  \multicolumn{1}{l|}{Methods}   & 
  \multicolumn{1}{c}{Reward} & 
  \multicolumn{1}{c}{$\Delta$W} & \multicolumn{1}{c|}{ error} & 
  \multicolumn{1}{c}{Reward} & 
  \multicolumn{1}{c}{$\Delta$W} & \multicolumn{1}{c|}{error} &
  \multicolumn{1}{c}{Reward} & 
  \multicolumn{1}{c}{$\Delta$W} & \multicolumn{1}{c}{error} \\  
  \midrule
  \multirow{1}{*}{Naive-S}
  & \textbf{536.0} &\textbf{89.1}& {0.489} & 1324.9 &478.4 &0.496 & 788.9 & 788.5 & 0.499\\
  \multirow{1}{*}{Naive-Y} 
  & 531.0 & 84.1 & \textbf{0.478} & 2099.1 & 863.0 & 0.409 & 1568.1 & 1557.8 & 0.410  \\
  \multirow{1}{*}{Ours}
  &  531.8 & 84.9 & 0.481 &   \textbf{2127.2} &\textbf{877.4} &\textbf{0.406}  & \textbf{1595.4} &\textbf{1585.1}& \textbf{0.407} \\
  \midrule
  \midrule
\multicolumn{1}{c|}{ \textsc{JOBS}}  
   & \multicolumn{3}{c|}{Short-term metrics}  
   & \multicolumn{3}{c|}{\textbf{Balanced metrics}}
   & \multicolumn{3}{c}{Long-term metrics}\\
  \midrule
  \multicolumn{1}{l|}{Methods}   & 
  \multicolumn{1}{c}{Reward} & 
  \multicolumn{1}{c}{$\Delta$W} & \multicolumn{1}{c|}{ error} & 
  \multicolumn{1}{c}{Reward} & 
  \multicolumn{1}{c}{$\Delta$W} & \multicolumn{1}{c|}{error} &
  \multicolumn{1}{c}{Reward} & 
  \multicolumn{1}{c}{$\Delta$W} & \multicolumn{1}{c}{error} \\  
  \midrule
  \multirow{1}{*}{Naive-S}
  & \textbf{1699.0} & \textbf{424.2} & \textbf{0.464} & 2858.0 & 410.8 & 0.481 & 1159.0 & -9.0 & 0.502 \\
  \multirow{1}{*}{Naive-Y}  
  & 1610.2 & 335.4 & 0.503 & 2869.0 & 380.8 & 0.478 & \textbf{1258.8} & \textbf{90.8} & \textbf{0.477}\\
  \multirow{1}{*}{Ours}
  & 1657.0 & 382.2 & 0.483 & \textbf{2885.6} & \textbf{412.5} & \textbf{0.475} & 1228.6 & 60.6 & 0.484 \\
  \bottomrule
  \end{tabular}
  \end{sc}
\label{tab:fixed_missing02}%
\end{table*}%

\begin{table*}[t!]
\centering
\caption{Comparison of the baselines, \textsc{Naive-S} (maximizing short-term rewards alone), \textsc{Naive-Y} (maximizing long-term rewards alone), and the proposed method in terms of the rewards, welfare changes, and policy errors on \textsc{IHDP} and \textsc{Jobs}. Higher reward/$\Delta$W and lower error mean better performance. The missing ratio is 0.3.}
\begin{sc}
\begin{tabular}{l | lll | lll | lll}
  \toprule
   \multicolumn{1}{c|}{ \textsc{IHDP}}  
   & \multicolumn{3}{c|}{Short-term metrics}  
   & \multicolumn{3}{c|}{\textbf{Balanced metrics}}
   & \multicolumn{3}{c}{Long-term metrics}\\
  \midrule
  \multicolumn{1}{l|}{Methods}   & 
  \multicolumn{1}{c}{Reward} & 
  \multicolumn{1}{c}{$\Delta$W} & \multicolumn{1}{c|}{ error} & 
  \multicolumn{1}{c}{Reward} & 
  \multicolumn{1}{c}{$\Delta$W} & \multicolumn{1}{c|}{error} &
  \multicolumn{1}{c}{Reward} & 
  \multicolumn{1}{c}{$\Delta$W} & \multicolumn{1}{c}{error} \\  
  \midrule
  \multirow{1}{*}{Naive-S}
  &  \textbf{534.1} & \textbf{87.2} & 0.495 & 1273.9 &451.9 &0.500  & 739.8 & 729.5 & 0.502 \\ %
  \multirow{1}{*}{Naive-Y}  
  & 533.7 & 86.8 &\textbf{0.479} & 2032.2 & 830.9 & 0.420 & 1498.5 & 1488.2 &0.420 \\ %
  \multirow{1}{*}{Ours}
  & 531.9 & 85.0 & {0.480} & \textbf{2051.8} & \textbf{839.8} & \textbf{0.417 }& \textbf{1520.0} & \textbf{1509.7} & \textbf{0.419}\\
  \midrule
  \midrule
\multicolumn{1}{c|}{ \textsc{JOBS}}  
   & \multicolumn{3}{c|}{Short-term metrics}  
   & \multicolumn{3}{c|}{\textbf{Balanced metrics}}
   & \multicolumn{3}{c}{Long-term metrics}\\
  \midrule
  \multicolumn{1}{l|}{Methods}   & 
  \multicolumn{1}{c}{Reward} & 
  \multicolumn{1}{c}{$\Delta$W} & \multicolumn{1}{c|}{ error} & 
  \multicolumn{1}{c}{Reward} & 
  \multicolumn{1}{c}{$\Delta$W} & \multicolumn{1}{c|}{error} &
  \multicolumn{1}{c}{Reward} & 
  \multicolumn{1}{c}{$\Delta$W} & \multicolumn{1}{c}{error} \\  
  \midrule
  \multirow{1}{*}{Naive-S}
  & \textbf{1701.4} & \textbf{426.6} &\textbf{0.467} & 2854.9 &419.3 & 0.482 &1153.5 & -14.4 & 0.502  \\ %
  \multirow{1}{*}{Naive-Y}  
  & 1605.5 &330.7 &0.512 & 2892.2& 390.1 & \textbf{0.474}& \textbf{1286.8} &\textbf{118.9} & \textbf{0.472}\\ %
  \multirow{1}{*}{Ours}
  & 1662.7 &387.9& 0.482&\textbf{2894.3}&\textbf{424.7} &\textbf{0.474}& 1221.6& 53.6 &0.485\\
  \bottomrule
  \end{tabular}
  \end{sc}
  \label{tab:fixed_missing03}%
\end{table*}%

\begin{table*}[t!]
\centering
\caption{Comparison of the baselines, \textsc{Naive-S} (maximizing short-term rewards alone), \textsc{Naive-Y} (maximizing long-term rewards alone), and the proposed method in terms of the rewards, welfare changes, and policy errors on \textsc{IHDP} and \textsc{Jobs}. Higher reward/$\Delta$W and lower error mean better performance. The missing ratio is 0.4.}
\begin{sc}
\begin{tabular}{l | lll | lll | lll}
  \toprule
   \multicolumn{1}{c|}{ \textsc{IHDP}}  
   & \multicolumn{3}{c|}{Short-term metrics}  
   & \multicolumn{3}{c|}{\textbf{Balanced metrics}}
   & \multicolumn{3}{c}{Long-term metrics}\\
  \midrule
  \multicolumn{1}{l|}{Methods}   & 
  \multicolumn{1}{c}{Reward} & 
  \multicolumn{1}{c}{$\Delta$W} & \multicolumn{1}{c|}{ error} & 
  \multicolumn{1}{c}{Reward} & 
  \multicolumn{1}{c}{$\Delta$W} & \multicolumn{1}{c|}{error} &
  \multicolumn{1}{c}{Reward} & 
  \multicolumn{1}{c}{$\Delta$W} & \multicolumn{1}{c}{error} \\  
  \midrule
  \multirow{1}{*}{Naive-S}
  & \textbf{536.8} &\textbf{89.9}& 0.490 & 1276.3& 454.5& 0.500 & 739.5 & 729.1& 0.501\\
  \multirow{1}{*}{Naive-Y} 
  & 529.7 & 82.8 & \textbf{0.481} & 1957.7 & 791.6 & 0.430 & 1428.0 & 1417.6 & 0.431  \\
  \multirow{1}{*}{Ours}
  & 529.6 & 82.7 & 0.483& \textbf{2017.1} & \textbf{821.2} & \textbf{0.421} & \textbf{1487.5} & \textbf{1477.2} & \textbf{0.421} \\
  \midrule
  \midrule
\multicolumn{1}{c|}{ \textsc{JOBS}}  
   & \multicolumn{3}{c|}{Short-term metrics}  
   & \multicolumn{3}{c|}{\textbf{Balanced metrics}}
   & \multicolumn{3}{c}{Long-term metrics}\\
  \midrule
  \multicolumn{1}{l|}{Methods}   & 
  \multicolumn{1}{c}{Reward} & 
  \multicolumn{1}{c}{$\Delta$W} & \multicolumn{1}{c|}{ error} & 
  \multicolumn{1}{c}{Reward} & 
  \multicolumn{1}{c}{$\Delta$W} & \multicolumn{1}{c|}{error} &
  \multicolumn{1}{c}{Reward} & 
  \multicolumn{1}{c}{$\Delta$W} & \multicolumn{1}{c}{error} \\  
  \midrule
  \multirow{1}{*}{Naive-S}
  & \textbf{1702.8} & \textbf{428.0} & \textbf{0.465} & 2861.7 & \textbf{423.5} & 0.480 & 1158.9 & -9.0 & 0.502 \\
  \multirow{1}{*}{Naive-Y}  
  & 1609.4 & 334.7 & 0.505 & 2890.1 & 391.0 & 0.474 & \textbf{1280.6} &\textbf{112.7} & \textbf{0.472} \\
  \multirow{1}{*}{Ours}
  & 1663.8 & 389.1 & 0.482 & \textbf{2891.5} & 418.9 & \textbf{0.473} &1227.7 &59.8& 0.484\\
  \bottomrule
  \end{tabular}
  \end{sc}
\label{tab:fixed_missing04}%
\end{table*}%

\begin{table*}[t!]
\centering
\caption{Comparison of the baselines, \textsc{Naive-S} (maximizing short-term rewards alone), \textsc{Naive-Y} (maximizing long-term rewards alone), and the proposed method in terms of the rewards, welfare changes, and policy errors on \textsc{IHDP} and \textsc{Jobs}. Higher reward/$\Delta$W and lower error mean better performance. The missing ratio is 0.5.}
\begin{sc}
\begin{tabular}{l | lll | lll | lll}
  \toprule
   \multicolumn{1}{c|}{ \textsc{IHDP}}  
   & \multicolumn{3}{c|}{Short-term metrics}  
   & \multicolumn{3}{c|}{\textbf{Balanced metrics}}
   & \multicolumn{3}{c}{Long-term metrics}\\
  \midrule
  \multicolumn{1}{l|}{Methods}   & 
  \multicolumn{1}{c}{Reward} & 
  \multicolumn{1}{c}{$\Delta$W} & \multicolumn{1}{c|}{ error} & 
  \multicolumn{1}{c}{Reward} & 
  \multicolumn{1}{c}{$\Delta$W} & \multicolumn{1}{c|}{error} &
  \multicolumn{1}{c}{Reward} & 
  \multicolumn{1}{c}{$\Delta$W} & \multicolumn{1}{c}{error} \\  
  \midrule
  \multirow{1}{*}{Naive-S}
  & \textbf{535.2} & \textbf{88.3} & 0.492 & 1308.9 & 470.0 & 0.498 & 773.7 & 763.4 & 0.499 \\
  \multirow{1}{*}{Naive-Y} 
  & 530.2 & 83.3 & \textbf{0.479} & 1977.9 & 801.9 & 0.429 & 1447.7 & 1437.3 & 0.429 \\
  \multirow{1}{*}{Ours}
  & 529.8 & 82.8 & 0.485 & \textbf{2003.4} & \textbf{814.5} & \textbf{0.423} &\textbf{1473.6 }& \textbf{1463.3} &\textbf{0.425} \\
  \midrule
  \midrule
\multicolumn{1}{c|}{ \textsc{JOBS}}  
   & \multicolumn{3}{c|}{Short-term metrics}  
   & \multicolumn{3}{c|}{\textbf{Balanced metrics}}
   & \multicolumn{3}{c}{Long-term metrics}\\
  \midrule
  \multicolumn{1}{l|}{Methods}   & 
  \multicolumn{1}{c}{Reward} & 
  \multicolumn{1}{c}{$\Delta$W} & \multicolumn{1}{c|}{ error} & 
  \multicolumn{1}{c}{Reward} & 
  \multicolumn{1}{c}{$\Delta$W} & \multicolumn{1}{c|}{error} &
  \multicolumn{1}{c}{Reward} & 
  \multicolumn{1}{c}{$\Delta$W} & \multicolumn{1}{c}{error} \\  
  \midrule
  \multirow{1}{*}{Naive-S}
  & \textbf{1700.6} & \textbf{425.8} & \textbf{0.469} & 2844.3 & 413.7 &0.483 & 1143.7 & -24.3 & 0.504 \\
  \multirow{1}{*}{Naive-Y}  
  & 1601.9 & 32701 &  0.513 &  2877.2 &  380.8 &  0.478 &  \textbf{1275.3 }&  \textbf{107.4 }&  \textbf{0.472}\\
  \multirow{1}{*}{Ours}
  & 1656.8 &  382.0 &  0.487 & \textbf{2893.1} & \textbf{416.2}& \textbf{0.475} &  {1236.3}& 68.4 & 0.482 \\
  \bottomrule
  \end{tabular}
  \end{sc}
\label{tab:fixed_missing05}%
\end{table*}%

\begin{table*}[t!]
\centering
\caption{Comparison of the baselines, \textsc{Naive-S} (maximizing short-term rewards alone), \textsc{Naive-Y} (maximizing long-term rewards alone), and the proposed method in terms of the rewards, welfare changes, and policy errors on \textsc{IHDP} and \textsc{Jobs}. Higher reward/$\Delta$W and lower error mean better performance. The missing ratio is 0.6.}
\begin{sc}
\begin{tabular}{l | lll | lll | lll}
  \toprule
   \multicolumn{1}{c|}{ \textsc{IHDP}}  
   & \multicolumn{3}{c|}{Short-term metrics}  
   & \multicolumn{3}{c|}{\textbf{Balanced metrics}}
   & \multicolumn{3}{c}{Long-term metrics}\\
  \midrule
  \multicolumn{1}{l|}{Methods}   & 
  \multicolumn{1}{c}{Reward} & 
  \multicolumn{1}{c}{$\Delta$W} & \multicolumn{1}{c|}{error} & 
  \multicolumn{1}{c}{Reward} & 
  \multicolumn{1}{c}{$\Delta$W} & \multicolumn{1}{c|}{error} &
  \multicolumn{1}{c}{Reward} & 
  \multicolumn{1}{c}{$\Delta$W} & \multicolumn{1}{c}{error} \\  
  \midrule
  \multirow{1}{*}{Naive-S}
  & \textbf{536.4} & \textbf{89.5} & 0.490 & 1335.7 & 484.0 & 0.494 & 799.3 & 788.9 & 0.495\\
  \multirow{1}{*}{Naive-Y} 
  & 530.1 & 83.2 & \textbf{0.484} & 1908.3 & 767.1 & 0.433 & 1378.2 & 1367.9 & 0.434 \\
  \multirow{1}{*}{Ours}
  & 530.6 & 83.7 & 0.487 & \textbf{1941.2 }& \textbf{783.8} & \textbf{0.431} & \textbf{1410.6} & \textbf{1400.3} & \textbf{0.431 }\\
  \midrule
  \midrule
\multicolumn{1}{c|}{ \textsc{JOBS}}  
   & \multicolumn{3}{c|}{Short-term metrics}  
   & \multicolumn{3}{c|}{\textbf{Balanced metrics}}
   & \multicolumn{3}{c}{Long-term metrics}\\
  \midrule
  \multicolumn{1}{l|}{Methods}   & 
  \multicolumn{1}{c}{Reward} & 
  \multicolumn{1}{c}{$\Delta$W} & \multicolumn{1}{c|}{ error} & 
  \multicolumn{1}{c}{Reward} & 
  \multicolumn{1}{c}{$\Delta$W} & \multicolumn{1}{c|}{error} &
  \multicolumn{1}{c}{Reward} & 
  \multicolumn{1}{c}{$\Delta$W} & \multicolumn{1}{c}{error} \\  
  \midrule
  \multirow{1}{*}{Naive-S}
  & \textbf{1698.1} & \textbf{423.3} & \textbf{0.465} & 2848.5 & 414.5 & 0.484 & 1150.4 & -17.6 & 0.505\\
  \multirow{1}{*}{Naive-Y}  
  & 1596.2 & 321.4 & 0.513 & 2869.2 & 373.9 & 0.480 & \textbf{1273.0} & \textbf{105.1 }& \textbf{0.474} \\
  \multirow{1}{*}{Ours}
  & 1663.0 & 388.2 & 0.485 & \textbf{2892.6} & \textbf{419.0} & \textbf{0.473} & 1229.6 & 61.6 & 0.483 \\
  \bottomrule
  \end{tabular}
  \end{sc}
\label{tab:fixed_missing06}%
\end{table*}%

\begin{figure}[t!] \label{fig:difftime0104}
\centering
\subfigure[Ratio 0.1 on \textsc{IHDP}]{
\includegraphics[width=0.23\textwidth]{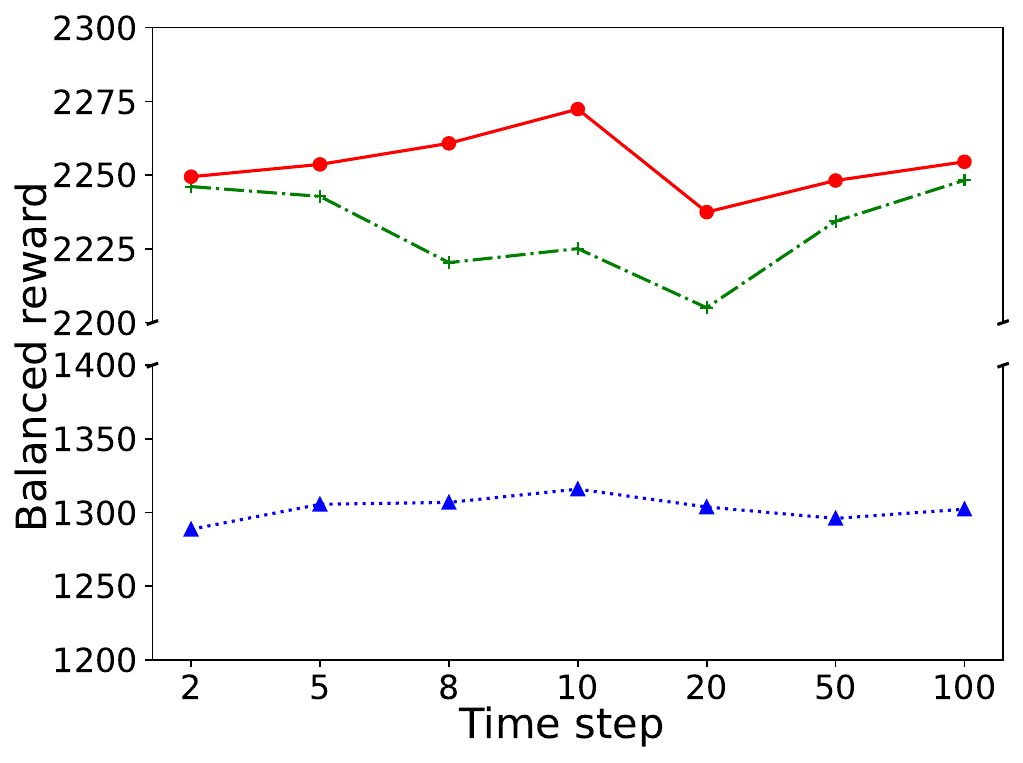}} 
\subfigure[Ratio 0.1 on \textsc{JOBS}]{
\includegraphics[width=0.23\textwidth]{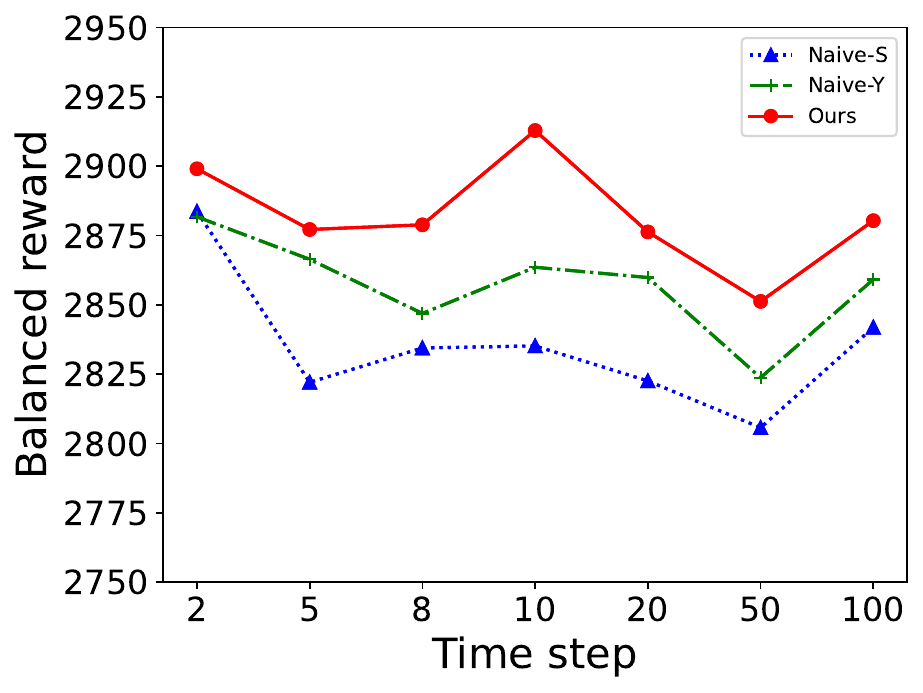}}
\subfigure[Ratio 0.4 on \textsc{IHDP}]{
\includegraphics[width=0.23\textwidth]{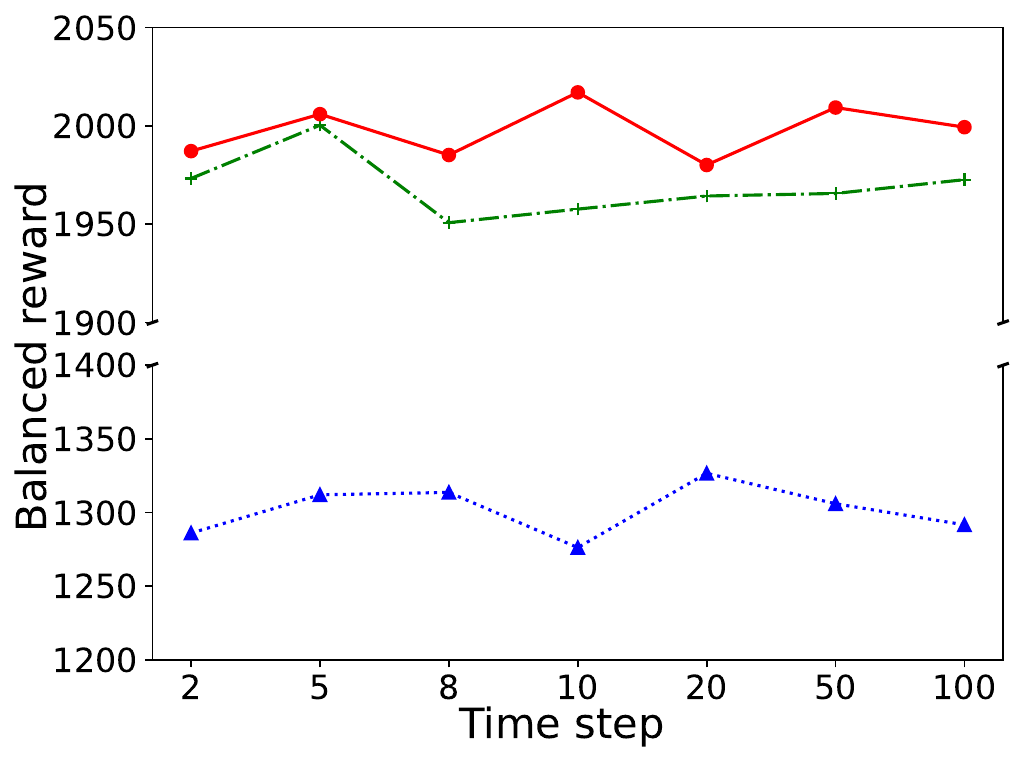}} 
\subfigure[Ratio 0.4 on \textsc{JOBS}]{
\includegraphics[width=0.23\textwidth]{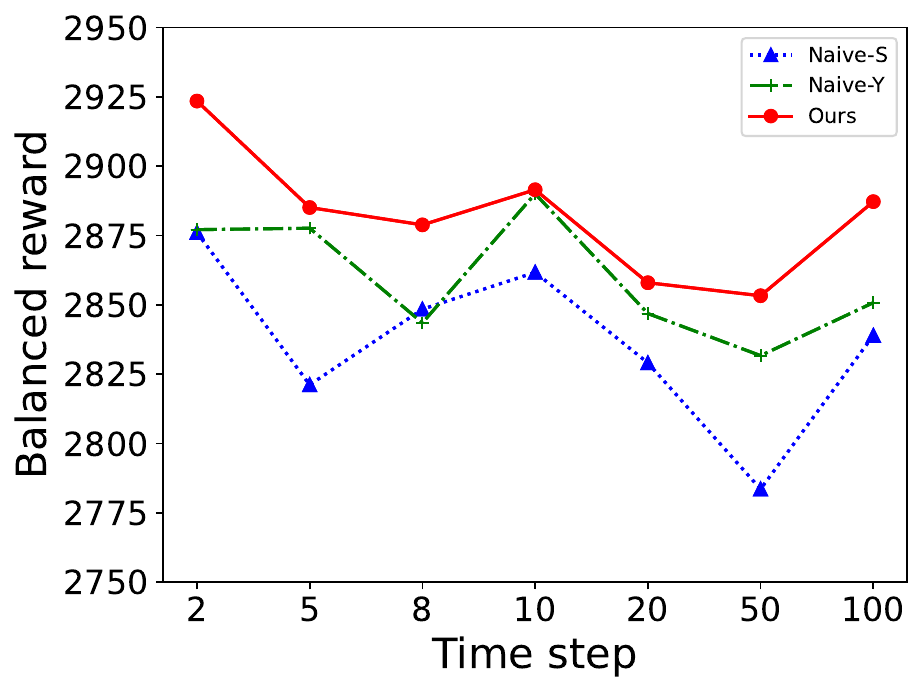}}
\caption{Comparison of \textsc{Naive-S}, \textsc{Naive-Y} and our method with different time steps and other fixed missing ratios $\{0.2, 0.4\}$ on \textsc{IHDP} and \textsc{JOBS}.}
\end{figure}

\newpage 
\section{Data Generation Details of $Y$ for $S \indep Y | X$}
\label{app:uncorr_data}
To generate long-term outcomes $Y_{i}(0)$ and $Y_{i}(1)$ such that $S \indep Y | X$, for $\textsc{IHDP}$, we set the initial value at time step $0$ as $Y_{0, i}(0) \sim \mathrm{Bern}(\sigma(w_0X_i+\epsilon_{0, i})), S_i(1)  \sim  \mathrm{Bern}(\sigma(w_1X_i+\epsilon_{1, i}))$, other than $Y_{0, i}(0) = S_i(0), Y_{0, i}(1) = S_i(1)$. 
For $\textsc{JOBS}$, we generate $Y$ with $Y \indep S | X$ in the following way, $Y_{t, i}(0)  \sim  \mathrm{Bern}(\sigma(\beta_0X_i))+ \epsilon_{0, i},
Y_{t, i}(1)  \sim  \mathrm{Bern}(\sigma(\beta_1X_i)) + \epsilon_{1, i}$.
Parameter values remain unchanged unless specified.
Both ways for two datasets break the correlated relationships between $S$ and $Y$ given $X$. 

\end{document}